\documentclass[12pt]{article}

\usepackage[utf8]{inputenc}
\usepackage[T1]{fontenc}
\usepackage{amsmath,amssymb,amsthm,amsfonts}
\usepackage{mathtools}
\usepackage{nicefrac}
\usepackage{xcolor}
\usepackage{float}
\usepackage{bm}
\usepackage{natbib}
\usepackage[a4paper,margin=1in]{geometry}
\usepackage{algorithm}
\usepackage{algpseudocode}
\usepackage{appendix}
\usepackage{graphicx}
\usepackage{caption}
\usepackage{subcaption}
\usepackage{booktabs}
\usepackage{longtable}
\usepackage{url}
\usepackage{hyperref}
\usepackage{pifont}
\newcommand{\cmark}{\ding{51}}
\newcommand{\xmark}{\ding{55}}

\newtheorem{theorem}{Theorem}
\newtheorem{lemma}[theorem]{Lemma}

\newtheorem{proposition}[theorem]{Proposition}

\newtheorem{assumption}{Assumption}
\newtheorem{remark}{Remark}

\theoremstyle{plain}      

\title{\textbf{Dynamic Treatment on Networks}}

\author{
Bengusu Nar \footnote{Bengusu (Bengüsu) Nar is a doctoral student in Econometrics and Statistics at the Booth School of Business of the University of Chicago},  Jiguang Li \footnote{Jiguang Li is a doctoral student in Econometrics and Statistics at the Booth School of Business of the University of Chicago}, Veronika Ro\v{c}kov\'{a} \footnote{Veronika Ro\v{c}kov\'{a} is  Bruce Lindsay Professor
of Econometrics and Statistics
in the Wallman Society of Fellows at the Booth School of Business of the University of Chicago} , and Panos Toulis \footnote{Panos (Panagiotis) Toulis is a Professor of Econometrics and Statistics, and John E. Jeuck Faculty Fellow at the Booth School of Business of the University of Chicago}
}
\begin{document}
\maketitle
\begin{abstract}
In networks, effective dynamic treatment allocation requires deciding both whom
to treat and also when, so as to amplify policy impact through spillovers. An early
intervention at a well-connected node can trigger cascades that change which
nodes are worth targeting in the next period. Existing treatment strategies under
network interference are largely static while dynamic treatment frameworks typically
ignore network structure altogether. We integrate these perspectives and propose
Q-Ising, a three-stage pipeline that (i) estimates network adoption dynamics via a Bayesian dynamic Ising model from a single observed panel, (ii) augments treatment adoption histories with continuous posterior latent states, and (iii) learns a dynamic policy via offline reinforcement learning. The Bayesian mechanism enables uncertainty
quantification over dynamic decisions, yielding posterior ensemble policies with
interpretable spillover estimates. We provide a finite-sample regret upper bound that decomposes into standard offline-RL uncertainty, network abstraction error, and first stage error in Ising state estimation. We apply our method to data from Indian village microfinance networks and synthetic stochastic block models under simulated heterogeneous susceptible-infected-susceptible (SIS) dynamics and demonstrate that adaptive targeting outperforms static centrality benchmarks. 
\end{abstract}


\section{Introduction}
\label{sec:intro}

When a planner makes decisions about a dynamic policy on a network, the central
problem is not only \emph{whom} to treat, but also \emph{when}. Under network
interference, the order in which units receive treatment determines which nodes
will spread spillovers first and how those will compound over time. A
policy that selects the right nodes to treat but ignores their sequencing can perform
strictly worse than one that orders strategically.

To see why ordering matters, consider a marketing campaign promoting a product
on a social network with a limited budget split across multiple periods. If the planner treats influential users
first, the second-period action changes: the planner can now target neighbors
who were exposed but have not yet adopted, or can start fresh elsewhere in the network.
The optimal action in period two depends on what period one achieved. Similar
sequencing problems arise in viral marketing~\citep{kempe2003maximizing, domingos2001mining}, platform engagement campaigns, and public health
interventions \citep{bubar2021model, buckner2021dynamic}.

The core difficulty in this setting is that the planner \emph{cannot experiment} and
is limited to a single observational trajectory under a historical policy. This
rules out online influence maximization (IM) \citep{kempe2003maximizing} algorithms that require oracle access to the diffusion mechanism \citep{singh2022influence}. Also, due to dynamic decision-making aspect, static policies are usually not optimal. To address this problem, the planner needs two things: first, an adaptive model of how the network behaves; second, a way to use that model to choose interventions sequentially.

For the first part, we use a dynamic Ising model \citep{yang1992glauber}. This model is designed for estimating the probability of each node's next state conditional on the current state of itself, its neighbors and the past intervention. Unlike equilibrium Ising models that require intractable partition functions, the dynamic formulation admits tractable node-wise likelihoods. For the second part, we use offline reinforcement learning (RL) which is a method for learning dynamic decision rules from historical data, without needing experimentation \citep{levine2020offline}. As a key contribution, we treat the estimated conditional probabilities as latent states for offline RL. Lastly, to assess the uncertainty of these dynamic decisions, we provide an ensemble framework for uncertainty quantification. This three-stage process provides a unified framework for dynamic policy under network interference. 

Related work spans multiple fields. Several approaches learn welfare-maximizing policies for single-period interventions \citep{kitagawa2018empirical, viviano2025policy}, and network-aware targeting based on network topologies like centrality and the friendship paradox \citep{banerjee2013diffusion, kempe2003maximizing, christakis2010social, kim2015social, chen2009efficient, liu2017fast}. Dynamic treatment regime methods \citep{murphy2001marginal,chakraborty2014dynamic, hu2025optimal, adusumilli2019dynamically, kitagawa2022policy} provide tools for dynamic decision-making but do not use network structures. Restless multi-armed bandits provide a framework for budget-constrained dynamic policy \citep{whittle1988restless, weber1990index, mate2020collapsing} but assume independently evolving units. Recent extensions embed arms in networks, allowing an intervention to benefit their neighbors \citep{herlihy2023networked, ou2022networked, vaswani2015influence, gleich2025scalable}, but these approaches assume known spillover mechanisms or repeated experimentation. Similarly, some works require Nash equilibrium and known dynamics \citep{kitagawa2023individualized}. Recent graph neural network (GNN) methods achieve strong empirical performance but also require online experimentation or known dynamics \citep{manchanda2020gcomb, sun2018multi, meirom2021controlling, feng2024influence}. A comparison among some of these methods is provided in Table \ref{tab:comparison} in Appendix ~\ref{app:tables}.

We propose \textbf{Q-Ising}, combining Bayesian dynamic Ising 
inference with offline RL for sequential decision-making under 
network interference. Methodologically, we transform the intractable 
problem of policy learning under network interference into a standard 
offline RL problem by treating estimated network dynamics as latent states. Theoretically, we give a regret upper bound for finite-horizon Q-Ising under pessimistic offline learning and show that the bound decomposes into standard offline RL uncertainty, network abstraction and first stage errors.

We employ Q-Ising on susceptible-infected-susceptible (SIS) 
dynamics \citep{kermack1927contribution, hethcote2000mathematics}, a widely used model of 
relapsing contagion \citep{bass1969new, jackson2007relating, bohner2016sis} on two regimes. First one is a stochastic block model (SBM) 
where the most influential nodes form a smaller community that is 
not identifiable by degree alone, making the setting adversarial 
for centrality-based methods. In this design, Q-Ising identifies 
the high-influence group from offline data alone and shifts the treatment budget adaptively as organic spread becomes 
self-sustaining. Second, we demonstrate Q-Ising's practical value on real microfinance networks from Karnataka, India \citep{banerjee2013diffusion}, where SIS dynamics are simulated on the empirical adjacency matrices. Across both regimes, Q-Ising matches or improves over the best baseline, provides interpretable coefficient estimates revealing the underlying dynamics, and quantifies uncertainty over recommended actions 
at each decision point. Neither of these features are available in GNN-based influence maximization approaches \citep{manchanda2020gcomb, sun2018multi, meirom2021controlling, feng2024influence}.

The rest of the paper is as follows: First, we formally define the problem in Section~\ref{sec:framework}. Then we introduce our methodology in Section \ref{sec:method}.  In Section ~\ref{sec:theory}, we give the assumptions for deriving the pessimistic suboptimality and present the regret bound. Finally, we  conduct our experiments in Section \ref{sec:experiments}.

\section{Framework and Problem Formulation}
\label{sec:framework}

We study dynamic treatment on a fixed network with binary outcomes.  Throughout, $y_{i,t}=1$ denotes adoption of a product or behavior by node $i$ in period $t$; adopted units may later disengage, so the planner must decide not only whom to treat but also when.

\paragraph{Setting and Timing.}
There are $N$ nodes connected by a fixed, undirected, observed network
$M\in\{0,1\}^{N\times N}$, where $M_{ij}=1$ if $i$ and $j$ are linked
and $\mathcal N_i=\{j:M_{ij}=1\}$ denotes the neighbors of $i$. Each
node has fixed observed features $\mathbf x_i\in\mathbb R^{d_x}$. At each period
$t$, node $i$ has a conditional adoption probability
$l_{i,t}\in[0,1]$ under the realized history and treatment, and
realizes $y_{i,t}\sim\mathrm{Bernoulli}(l_{i,t})$. We write
$\mathbf y_t=(y_{1,t},\ldots,y_{N,t})^\top$.

We reserve $t=1,\ldots,T_{\mathrm{train}}$ for periods in the
training panel. For policy evaluation, we re-index the target
deployment horizon by $h=1,\ldots,H$. Before a period-$t$ decision, the planner observes the pre-action history 
\begin{equation*}
 Z_t := (\mathbf y_{0:t-1},\,a_{1:t-1},\,\mathbf X,\,M),
  \label{eq:pre_action_history}
\end{equation*}
 with $Z_1 = (\mathbf y_0, \mathbf X, M)$. The sigma-algebra generated by $Z_t$ is the filtration $\mathcal{F}_t$, but we use $Z_t$ directly to keep the notation explicit. Each period allocates one treatment to a node, $a_t\in[N]$. We reserve $a_t=\emptyset$ only as an auxiliary no-intervention action used later for counterfactual state construction.

\paragraph{Data.}
The planner observes one trajectory of node-level outcomes and
treatments,
\begin{equation*}
  \mathcal D
  =
  \left\{\mathbf y_0,\,(a_t,\mathbf y_t)_{t=1}^{T_{\mathrm{train}}}\right\},
  \label{eq:data}
\end{equation*}
collected under a historical policy. This is the practically
important regime of a single long panel from one real network, with
$M$ and $\mathbf X$ fixed and known throughout. We retain node-level
actions because Stage~1 estimates direct treatment and neighbor-spillover effects.

\paragraph{Bin-level policy class.}
For large $N$, node-specific targeting is both statistically and 
practically challenging. Under roughly uniform logging, each node receives only
$O(T_{\mathrm{train}}/N)$ treatment observations, which is insufficient
for reliable policy learning under interference. Beyond data coverage, 
treatment rules defined over covariate groups rather than individuals 
are more interpretable, auditable, and aligned with fairness 
requirements common in policy applications 
\citep{kitagawa2018empirical, kitagawa2023vaccinated, viviano2024fair}. 

We therefore partition nodes into $K$ disjoint bins 
$\{\mathcal{B}_1,\ldots,\mathcal{B}_K\}$ based on covariates 
$\mathbf{X}$ and/or network structure $M$, for example communities 
identified by spectral clustering or demographic strata such as young 
married households. The choice of $K$ trades off statistical coverage 
against targeting granularity: each bin accumulates 
$O(T_{\mathrm{train}}/K)$ treatment observations. $K$ should be 
chosen so that within-bin nodes are reasonably homogeneous with respect 
to both covariates and network position.

We distinguish two levels of action: the planner selects a bin 
$b_t \in \{1,\ldots,K\}$ each period, and one node is drawn uniformly 
at random from the selected bin as the realized treatment:
\begin{equation}
   b_t=\pi_h(Z_t),
  \qquad
  a_t\sim\mathrm{Unif}(\mathcal B_{b_t}).
    \label{eq:action}
\end{equation}
Historical node actions are mapped to bin actions  when constructing the offline RL transitions. The target deployment rule in~\eqref{eq:action} imposes uniform within-bin randomization; However, the historical within-bin selector may differ, a distinction accounted for in our regret theory.

\paragraph{Welfare Maximization} Let $\Pi_Z$ denote the class of admissible full-history bin policies $\pi=(\pi_1,\ldots,\pi_H)$, where $\pi_h:\mathcal Z_h\to[K]$ maps the pre-action history $Z_h$ to a bin decision. The per-period reward is the network-wide adoption rate,
\begin{equation}
  r_h = \frac{1}{N}\sum_{i=1}^N y_{i,h}.
  \label{eq:reward}
\end{equation}
For $\pi\in\Pi_Z$, its policy value, or cumulative welfare, is the
expected reward along the full trajectory induced by deploying $\pi$:
\begin{equation*}
  W_H^Z(\pi;Z_1) := \mathbb E^\pi\!\left[ \sum_{h=1}^{H} r_h \,\middle|\, Z_1 \right],
  \label{eq:welfare}
\end{equation*}
where the expectation is over the stochastic network evolution and the
within-bin randomization generated by the policy. The full-history
population target is
\begin{equation}
  \pi_Z^\star
  \in
  \arg\max_{\pi\in\Pi_Z} W_H^Z(\pi;Z_1).
  \label{eq:pistar}
\end{equation}

This objective is inherently dynamic: treating a node at time $t$
changes subsequent adoption probabilities through persistence and
network spillovers, so the value of a current bin decision depends on
how it shapes future network states. Thus myopic targeting needs not be optimal. Formally, let a greedy rule select
$b_h\in\arg\max_{b\in[K]}\mathbb E[r_h\mid Z_h,b_h=b]$ at each stage.
\begin{proposition}[Greedy Suboptimality]
\label{prop:heur_sub}
The greedy policy that maximizes the immediate reward is not necessarily optimal.
\end{proposition}
Appendix~\ref{app:greedy_sub} proves the proposition by counterexample
under basic SIS dynamics.

\section{Methodology}
\label{sec:method}

Q-Ising proceeds in three stages. First, we estimate a dynamic Ising model from the panel $\mathcal D$. Second, we use the fitted model to construct low-dimensional Q-Ising states. Third, we apply offline RL to learn a dynamic bin-level policy over these states.

\subsection{Stage 1: Dynamic Ising Inference}
\label{sec:ising}

Standard Equilibrium Ising model requires computing an intractable partition
function and symmetric interactions. We instead use a dynamic conditional
model, analogous to logistic pseudo-likelihood estimation for Ising
models~\citep{ravikumar2010ising}, which allows asymmetric influence
and temporal dependence. Define the linear predictor for node $i$ in bin $\mathcal{B}_k$ at time $t$:
\begin{align}
  \eta_{i,t} (a_t; \boldsymbol{\theta}_i) \;=\;
    \beta_{0,k}
    + \beta_{1,k}\,\mathbf{1}_{a_{t}=i}
    + \beta_{2,k}\,y_{i,t-1}
    + \beta_{3,k}\,\mathbf{1}_{a_{t}\in\mathcal{N}_i}
    + \sum_{j\in\mathcal{N}_i}\gamma_{k,m_j}\,y_{j,t-1}.
  \label{eq:linear_pred}
\end{align}

Here $m_j$ represents the bin node $j$ belongs. Although $M$ is undirected, the interaction parameters need not satisfy $\gamma_{k,m_j} = \gamma_{k,m_j}$: node $j$ may strongly influence node $i$ without the reverse being true.  This asymmetry  captures the influencer structure common in social networks, where a well-connected household may drive adoption among neighbors without being equally susceptible to peer influence itself. The standard equilibrium Ising model would not be able to capture this important feature. 

The parameter vector $\boldsymbol{\theta}_i = [\beta_{0,k}, \beta_{1,k},
\beta_{2,k}, \beta_{3,k},
\boldsymbol{\gamma}_{\mathcal{N}(i)}^\top]^\top$ collects all
parameters for node $i$, and $\boldsymbol{\theta} =
[\boldsymbol{\theta}_1,\ldots, \boldsymbol{\theta}_N]^\top$. Each parameter has a natural
interpretation: $\beta_{1,k}$ captures the direct effect of treating node in bin $\mathcal{B}_k$; $\beta_{2,k}$ captures persistence of past adoption;
$\beta_{3,k}$ captures spillover from treating a neighbor; and
$\gamma_{k,m_j}$ captures peer influence
from neighbors in bin $m_j$ on nodes in bin $k$. Node-level coefficients
are a direct extension obtained by replacing bin indices with node
indices when the panel is sufficiently long.

Let $\sigma(z) = (1+e^{-z})^{-1}$ denote the logistic function.
The probability of node $i$ adopting after taking action $a_{t}$ is given by
\begin{equation*}
    \mathbb{P}(y_{i,t}=1\mid Z_t,\, a_{t}, \, \boldsymbol{\theta}_i) = \sigma(\eta_{i,t}(a_t; 
    \boldsymbol{\theta}_i)).
  \label{eq:ising}
\end{equation*}

Given $Z_t$, $a_{t}$, and $\boldsymbol{\theta}$, outcomes are independent Bernoulli draws across nodes. Since the process is also first-order Markov,the likelihood over the panel factorizes as
\begin{equation}
  \mathcal{L}(\boldsymbol{\theta}) = \prod_{i=1}^{N}\prod_{t=1}^{T}
    \frac{\exp(y_{i,t}\cdot\eta_{i,t}(a_t; \boldsymbol{\theta}_i))}{1+\exp(\eta_{i,t}(a_t; \boldsymbol{\theta}_i))},
  \label{eq:likelihood}
\end{equation}
enabling efficient likelihood-based  inference.

\begin{remark}[Connection to standard Ising]
Under synchronous updating, which is the relevant case
for policy applications where all nodes respond simultaneously, the
stationary distribution of the dynamic Ising model cannot in general be expressed as an element-wise Gibbs measure, making the dynamic formulation the appropriate choice for our setting. See Appendix~\ref{app:ising_proofs} for proof.
\end{remark}

\paragraph{Priors and sparsity.}
For interaction parameters we impose a continuous spike-and-slab
prior encouraging sparsity in peer effects \citep{rovckova2014emvs, george1993variable}:
\begin{equation*}
  \gamma_{i,j}\mid z_{ij} \sim
    (1-z_{ij})\,\mathcal{N}(0,v_0)
    + z_{ij}\,\mathcal{N}(0,v_1),
  \quad v_0 \ll v_1,\quad
  z_{ij} \sim
    \mathrm{Bernoulli}\!\left(\tfrac{c}{|\mathcal{B}_{m_i}|}\right),
  \label{eq:prior}
\end{equation*}
with hyperparameters $v_0 = 0.01$, $v_1 = 10.0$, and $c = 1.0$.
The bin-size dependent inclusion probability encodes the expectation
that the number of influential bins scales sublinearly with the average
degree of the affected bin. For baseline, treatment, and persistence parameters,
$\beta_{l,k} \sim \mathcal{N}(0,\tau^2)$ with $\tau^2 = 10.0$.
Zero-centered priors ensure that when data are sparse, estimates
shrink toward $\hat{\boldsymbol{\theta}}_i \approx 0$, yielding
$\hat{l}_{i,t} \approx 0.5$: a principled uninformed baseline.

\paragraph{Estimation: EMVS and MCMC.}
We offer two complementary approaches.
\textbf{EMVS} \citep{rovckova2014emvs} alternates between computing
posterior inclusion probabilities and solving a weighted penalized
logistic regression, requiring only convex optimization per iteration
and typically converging in 2--3 steps; it is particularly attractive
for large $N$ where speed is essential. When posterior uncertainty is
needed, we sample from $p(\boldsymbol{\theta}_i\mid\mathcal{D})$ via
\textbf{Hamiltonian Monte Carlo} \cite{JMLR:v15:hoffman14a}, producing $P$ posterior draws
$\{\boldsymbol{\theta}_i^{(p)}\}_{p=1}^P$ per node, which serve as the basis for
the ensemble policy in Section~\ref{sec:ensemble}.

\paragraph{State construction.}
Given $\hat{\theta}_i$, we construct \emph{latent states} that summarize the
network's adoption landscape at each period.  For node $i$, define the estimated no-intervention adoption probability
\begin{equation}
  \widehat l^0_{i,t} := \sigma\!\left(\eta_{i,t}(\emptyset;\widehat{\boldsymbol{\theta}})\right),
  \label{eq:belief}
\end{equation}
where $a_t=\emptyset$ sets treatment indicators in \eqref{eq:linear_pred} to zero. Setting $a_t = \emptyset$ is important because it captures where the network is \emph{headed} absent the current intervention, serving as a forward-looking baseline. We aggregate both the estimated latent states and the observed outcomes
to the bin level:
\begin{equation}
   \bar l^0_{t,k} = \frac{1}{|\mathcal B_k|} \sum_{i\in\mathcal B_k}\widehat l^0_{i,t}, \qquad
  \bar y_{t-1,k} = \frac{1}{|\mathcal B_k|} \sum_{i\in\mathcal B_k}y_{i,t-1}, \qquad k\in[K].
  \label{eq:bin_state}
\end{equation}
Let $\bar{\mathbf l}^0_t=(\bar l^0_{t,1},\ldots,\bar l^0_{t,K})$ and
$\bar{\mathbf y}_{t-1}=(\bar y_{t-1,1},\ldots,\bar y_{t-1,K})$. The
Q-Ising state is
\begin{equation}
  s_t
  =
  \left(\bar{\mathbf l}^0_t,\bar{\mathbf y}_{t-1}\right)
  \in
  \mathcal S:=[0,1]^K\times[0,1]^K.
  \label{eq:state}
\end{equation}
The first component is forward-looking and model-based; the second is
the realized bin-level adoption profile before the current decision.

\subsection{Stage 2: Offline Q-Learning}
\label{sec:rl}

We construct the transitions $\mathcal{D}_{\mathrm{RL}}:=\{(s_t,\, b_t,\, r_t,\, s_{t+1})\}_{t=1}^{T_{\mathrm{train}}}$ from $\mathcal{D}$, where $s_t$ is defined in
eq.~\eqref{eq:state}, $b_t \in \{1,\ldots,K\}$ is the bin-level
action mapped from $a_t$, and $r_t$ is defined in
eq.~\eqref{eq:reward}. Learning over $K$ bin actions is statistically more stable than
learning over $N$ node actions, since each bin aggregates many 
treatments.

For the empirical implementation, we learn a discounted Q-function $Q:\mathcal S\times[K]\to\mathbb R$ with discount $\psi\in[0,1)$, which represents the expected cumulative  reward from selecting bin $b$ at state $s$ and following the optimal 
policy thereafter. The Q-function satisfies the Bellman target  \citep{sutton1998reinforcement} \begin{equation*}
  \widehat{\mathcal{T}} Q(s_t, b_t) \;=\;
  r_t + \psi\max_{b'\in\{1,\ldots,K\}} Q(s_{t+1}, b').
  \label{eq:bellman}
\end{equation*}
Since $\mathcal D_{\mathrm{RL}}$ is collected under a historical
policy, some state-action pairs are poorly supported. Standard
Q-learning can then overestimate unsupported actions. We therefore use
the pessimism principle from offline RL: unsupported actions should be
penalized rather than optimistically extrapolated. In experiments, we
use conservative Q-learning (CQL)~\citep{kumar2020conservative} via
\texttt{d3rlpy}~\citep{seno2022d3rlpy}; Section~\ref{sec:theory}
analyzes an idealized finite-horizon pessimistic variant of the value iteration (PEVI) version of Q-Ising under the
same pessimism principle as~\citep{jin2021pessimism}. In particular, the empirical CQL objective is
\begin{equation}
  \mathcal{L}_{\mathrm{CQL}}(Q) \;=\;
  \underbrace{\mathbb{E}_{(s,b,r,s')\sim\mathcal{D}_{\mathrm{RL}}}\!\left[
    \bigl(Q(s,b) - \widehat{\mathcal{T}}\bar{Q}\bigr)^2
  \right]}_{\text{Bellman error}}
  \;+\;
  \alpha\,\underbrace{\mathbb{E}_{(s,b)\sim\mathcal{D}_{\mathrm{RL}}}\!\left[
    \log\sum_{b'} \exp Q(s,b') - Q(s,b)
  \right]}_{\text{conservative penalty}},
  \label{eq:cql}
\end{equation}

where $\alpha > 0$ controls the strength of the conservative penalty.
The first term is the standard Bellman error, minimized over 
transitions observed in $\mathcal{D}_{\mathrm{RL}}$. The second term penalizes 
Q-values on all actions at observed states while pushing up Q-values 
on the actions actually taken in $\mathcal{D}_{\mathrm{RL}}$, mitigating 
distributional shift in offline policy learning \citep{levine2020offline}.

\subsection{Stage 3: Ensemble Policy and Uncertainty Quantification via Posterior Sampling}
\label{sec:ensemble}
Posterior sampling propagates first-stage uncertainty into the learned
policy. Each MCMC draw $\theta^{(p)}$ induces node-level estimates
$\widehat l^{0,(p)}_{i,t}$ via eq.\eqref{eq:belief}, bin-level aggregates $\bar l^{0,(p)}_{t,k}$ via eq.\eqref{eq:bin_state}, and hence a state
$s_t^{(p)}=(\bar{\mathbf l}^{0,(p)}_t,\bar{\mathbf y}_{t-1})$.
Note that $\bar{\mathbf{y}}_{t-1}$ is shared across all draws as it
depends only on observed data. Training a separate CQL network on
each draw produces an ensemble $\{\hat{\pi}^{(p)}\}_{p=1}^P$
reflecting parameter uncertainty.
\begin{algorithm}[ht]
\caption{Q-Ising Algorithm: Policy Learning via Dynamic Ising and Offline RL}
\label{alg:ensemble}
\begin{algorithmic}[1]
\Require Panel $\mathcal{D}$, adjacency $M$, bin partition
  $\{\mathcal{B}_1,\ldots,\mathcal{B}_K\}$, flag \texttt{ensemble} $\in \{\texttt{true},\texttt{false}\}$
\If{\texttt{ensemble}}
  \State Draw $\{ \boldsymbol{\theta}^{(p)}\}_{p=1}^P$ via HMC from
    $\pi(\boldsymbol{\theta}  \mid \mathcal{D}, M)$; set $\mathcal{P} \leftarrow \{1,\ldots,P\}$
\Else
  \State Fit $\hat{\boldsymbol{\theta}}$ via EMVS on $(\mathcal{D}, M)$;
    set  ${ \boldsymbol{\theta}}^{(1)} \leftarrow \hat{\boldsymbol{\theta}}$, $\mathcal{P} \leftarrow \{1\}$
\EndIf
\For{$p \in \mathcal{P}$}
  \State Compute $\widehat l^{0,(p)}_{i,t}
    =\sigma(\eta_{i,t}(\emptyset;\boldsymbol{\theta}_i^{(p)}))$ for all $i,t$
    via \eqref{eq:belief}
  \State Form $s_t^{(p)}=(\bar{\mathbf l}^{0,(p)}_t,\bar{\mathbf y}_{t-1})$
    using \eqref{eq:bin_state}--\eqref{eq:state}
  \State Construct transitions
    $\{(s_t^{(p)},\, b_t,\, r_t,\, s_{t+1}^{(p)})\}_{t=1}^{T-1}$
  \State Train $Q_p$ via \eqref{eq:cql}; set
    $\hat{\pi}^{(p)}(s) = \arg\max_{k}\,Q_p(s,k)$
\EndFor
\If{\texttt{ensemble}}
  \State \Return $\hat{\pi}^{\mathrm{ens}}(s) =
    \arg\max_{k}\sum_{p=1}^{P}
    \mathbf{1}\{\hat{\pi}^{(p)}(s)=k\}$
\Else
  \State \Return $\hat{\pi}(s) = \hat{\pi}^{(1)}(s)$
\EndIf
\end{algorithmic}
\end{algorithm}
Effectively, this posterior sampling helps recognizing potential lack of identification in the dynamic Ising model as Bellman optimality provides unique policies for each trained agent.  When most draws agree, the planner acts with
confidence; when votes are dispersed, the allocation is sensitive to
parameter uncertainty and warrants caution.  In practice,
$P=10$--$20$ posterior draws yield stable ensemble policies. 

\section{Theory} \label{sec:theory}
We provide a finite-sample regret guarantee for an idealized finite-horizon PEVI version of Q-Ising. The empirical algorithm in Section~\ref{sec:rl} uses CQL as a scalable implementation of the same pessimism principle; the theorem below analyzes the PEVI analogue ~\citep{jin2021pessimism}.

Let $\theta_0$ be the population Ising parameter and $\phi_{\theta}$ be the Q-Ising state map in \eqref{eq:belief}--\eqref{eq:state}. For a pre-action history $Z_h$, write $s_h^\star=\phi_{\theta_0}(Z_t)$ and $\widehat s_h=\phi_{\widehat\theta}(Z_h)$. For any policy $\pi$, define its oracle full history lift by $\pi_h^\circ(z)=\pi_h(\phi_{\theta_0}(z))$ Our comparator is the best oracle Q-Ising policy
$$\pi_S^\star \in \arg\max_{\pi\in\Pi_S} W_H^Z(\pi^\circ;Z_1), \qquad  \Pi_S = \{\pi=(\pi_1,\ldots,\pi_H):\pi_h:\mathcal S\to[K]\}.$$
This comparator is not the unrestricted full-history optimum
$\pi_Z^\star$ in \eqref{eq:pistar}, but the optimal Q-Ising policy.

Let $\widehat\pi$ be the finite-horizon pessimistic policy trained on
the Q-Ising offline transitions
$$\mathcal D_h=
  \{(\widehat s_{\tau,h},b_{\tau,h},r_{\tau,h},
      \widehat s_{\tau,h+1})\}_{\tau\in\mathcal I_h}, \qquad h=1,\ldots,H,$$     
where $\tau$ indexes the effective training blocks. Let $n_h=|\mathcal I_h|$, $n_{\mathrm{eff}}=\min_h n_h$, and
$n_{\log}=\max_h n_h$. For a fixed feature map
$\varphi:\mathcal S\times[K]\to\mathbb R^{d_\phi}$, with
$\|\varphi(s,b)\|_2\le1$, define the empirical design matrix
$$\Lambda_h= \lambda I_{d_\phi} + \sum_{\tau\in\mathcal I_h}
\varphi(\widehat s_{\tau,h},b_{\tau,h}) \varphi(\widehat s_{\tau,h},b_{\tau,h})^\top .$$

The main theorem analyzes the oracle deployment $\widehat\pi_h^\circ(z)=\widehat\pi_h(\phi_{\theta_0}(z))$ with regret
\begin{equation*}
  \operatorname{Reg}_H(\widehat\pi^\circ) := W_H^Z((\pi_S^\star)^\circ;Z_1)
  - W_H^Z(\widehat\pi^\circ;Z_1).
  \label{eq:oracle_regret}
\end{equation*}

The PEVI uncertainty of a policy $\pi$ under the training data aggregate dynamics is
\begin{equation*}
  \mathfrak U_{\mathcal D}^{\mu}(\pi) := \sum_{h=1}^{H}
  \mathbb E_{\mu}^{\pi} \left[
    \sqrt{
      \varphi(s_h,b_h)^\top
      \Lambda_h^{-1}
      \varphi(s_h,b_h)
    }
  \right],
  \label{eq:pevi_uncertainty}
\end{equation*}
where $\mathbb E_{\mu}^{\pi}$ denotes expectation over the bin-level
Q-Ising process induced by following $\pi$ while using the historical
within-bin selector. 

The proof relies on four regularity conditions, stated formally in
Appendix~\ref{subsec:ra-assumptions}. In words, historical actions must be valid interventions; the single network trajectory must yield effective
stage-wise regression blocks; the Q-Ising state must be an approximate
bin-level abstraction of the full network process; and the PEVI Bellman
regressions must have controlled projected misspecification. The
abstraction errors $\varepsilon_r^{\mathrm{agg}},\varepsilon_P^{\mathrm{agg}}$
measure the loss from replacing the full network by the oracle Q-Ising
state process. The binning errors $\varepsilon_r^{\mathrm{bin}},\varepsilon_P^{\mathrm{bin}}$
measure the mismatch between the target rule, which randomizes uniformly
within a selected bin, and the historical within-bin selector.
The quantities $\varepsilon_{\mathrm{lin}}$ and
$L_{\mathrm{st}}\varepsilon_\theta$ respectively capture projected
Bellman approximation error and first-stage Q-Ising state-estimation error. Finally, Let $\mathcal E_\theta$ denote the first-stage stability event that replacing oracle Q-Ising states $s_h^\star=\phi_{\theta_0}(Z_h)$ by estimated states $\widehat s_h=\phi_{\widehat\theta}(Z_h)$ induces controllable projected PEVI regression error. We assume $\mathbb P(\mathcal E_\theta)\ge 1-\delta_\theta$.

\begin{theorem} \label{thm:pevi_qising_compact}
Suppose the regularity conditions summarized above and formalized in
Appendix~\ref{subsec:ra-assumptions} hold. Choose
$$\beta = C_\beta \left[ H \sqrt{ d_{\phi}\log\!\left( \frac{H(1+n_{\log}/\lambda)}{\delta} \right)} +\sqrt{\lambda} W \right],$$
where $W$ bounds the linear Bellman coefficients and $C_\beta$ is a sufficiently large constant. Then, with probability
at least $1-\delta-\delta_\theta$,
$$\operatorname{Reg}_H(\widehat\pi^\circ) \le 2\beta\, \mathfrak U_{\mathcal D}^{\mu}(\pi_S^\star) +
  C \left[ \Delta_{\mathrm{abs}} + H\varepsilon_{\mathrm{lin}} + HL_{\mathrm{st}}\varepsilon_\theta \right],
$$
where $C>0$ is a universal constant and $\Delta_{\mathrm{abs}} = \Delta_{\mathrm{agg}} + \Delta_{\mathrm{bin}}$, with $\Delta_{\mathrm{agg}}:= H\varepsilon_r^{\mathrm{agg}}+ H^2\varepsilon_P^{\mathrm{agg}}$, and $\Delta_{\mathrm{bin}} =  H\varepsilon_r^{\mathrm{bin}} +  H^2\varepsilon_P^{\mathrm{bin}}$. 
\end{theorem}

The first term is the standard pessimistic offline-RL error that shrinks with larger effective sample size under standard coverage conditions. With in $\Delta_{\mathrm{abs}}$, $\Delta_{\mathrm{agg}}$ is the cost of compressing the full network history into bin-level Q-Ising states; $\Delta_{\mathrm{bin}}$ measures the mismatch between learning under the historical within-bin selector and deploying the target rule that randomizes uniformly within the selected bin. It vanishes to $0$ if the historical policy also randomizes uniformly within selected bins. The term $H\varepsilon_{\mathrm{lin}}$ is the cost of linear approximation for the Bellman equations, which vanishes under an exact linear Bellman model. The term $H L_{\mathrm{st}}\varepsilon_\theta$ is the first-stage Ising state-estimation cost, which can vanish under consistency of the dynamic Ising estimator.

In particular, if the Q-Ising abstraction is sufficient and exact, the historical within-bin selector matches the target uniform selector, and the first-stage and Bellman approximation errors
vanish, Theorem~\ref{thm:pevi_qising_compact} reduces to the usual PEVI
guarantee. Appendix~\ref{sec:ra-proof} gives the formal proof of the theorem. 

\section{Experiments}
\label{sec:experiments}
We evaluate Q-Ising across two network regimes: a  Stochastic Block Model (SBM) in Subsection \ref{subsec:sbm}  and microfinance networks from Karnataka, India \citep{banerjee2013diffusion} in Subsection \ref{subsec:villages}. In both cases, we simulate observational data from heteregenous, synthetic SIS dynamics. These dynamics are designed to be  adversarial for degree-based methods. The optimal strategy requires treating the community of highly susceptible nodes early which ignites an organic within-group spread, and then reallocating to other communities before saturation erases the marginal benefit of further treatments in the highly-susceptible group. A ranker that sorts nodes by degree will concentrate its budget on central but not necessarily susceptible nodes. Neither whom to treat nor when is
recoverable from degree statistics alone and requires adaptive decision making. 

All experiments run on an M1
chip, with Q-Ising and Plain DQN each requiring roughly $60$--$65$
seconds of wall-clock training time.
Ising parameters are estimated from this panel via EMVS; bin-level states $\bar{\mathbf{l}}_t$ are constructed via~\eqref{eq:bin_state} and take around $20-30$ seconds of wall clock time. We compare against five reference policies. Three are topology-only heuristics:  a degree-bin policy that iterates over the bins and selects the highest degree untreated node, a degree centrality policy, and LIR ~\citep{liu2017fast} which identifies local degree leaders to avoid the rich-club effect. The fourth, Plain DQN, follows our offline RL framework without the Ising augmentation, using the observational dynamics but no structural model. Lastly, we compare against a random bin policy which is also used to generate historical panel data.\footnote{Even though NEWM \citep{viviano2025policy} seems like a natural reference policy, it is not directly scalable to the large networks used in experiments and also does not take panel data as input and would require significant adaptation.}
\subsection{Experiments in Stochastic Block Model}
\label{subsec:sbm}

We simulate a stochastic block model based adjacency matrix with 500 nodes. There are four blocks with varying sizes, split approximately as 187-187-63-63. Within-block edge probability is $p_{\mathrm{in}} = 0.1$ and between-block probability is $p_{\mathrm{out}} = 0.01$. Spread rates are (0.010, 0.012, 0.1, 0.12) and churn rates are (0.4, 0.4, 0.2, 0.2). Small communities get high spread and low churn rates, representing "active" behavior. The planner observes $T_{\mathrm{Train}} = 100$ train periods from historical random-bin policy. The mean rewards of policies can be seen in Figure \ref{fig:improvement}, where test horizon is $H = 25$ and 50 independent tests are started from no adoption in the network.

\begin{figure}[ht]
    \centering
    \includegraphics[width=0.45\linewidth]{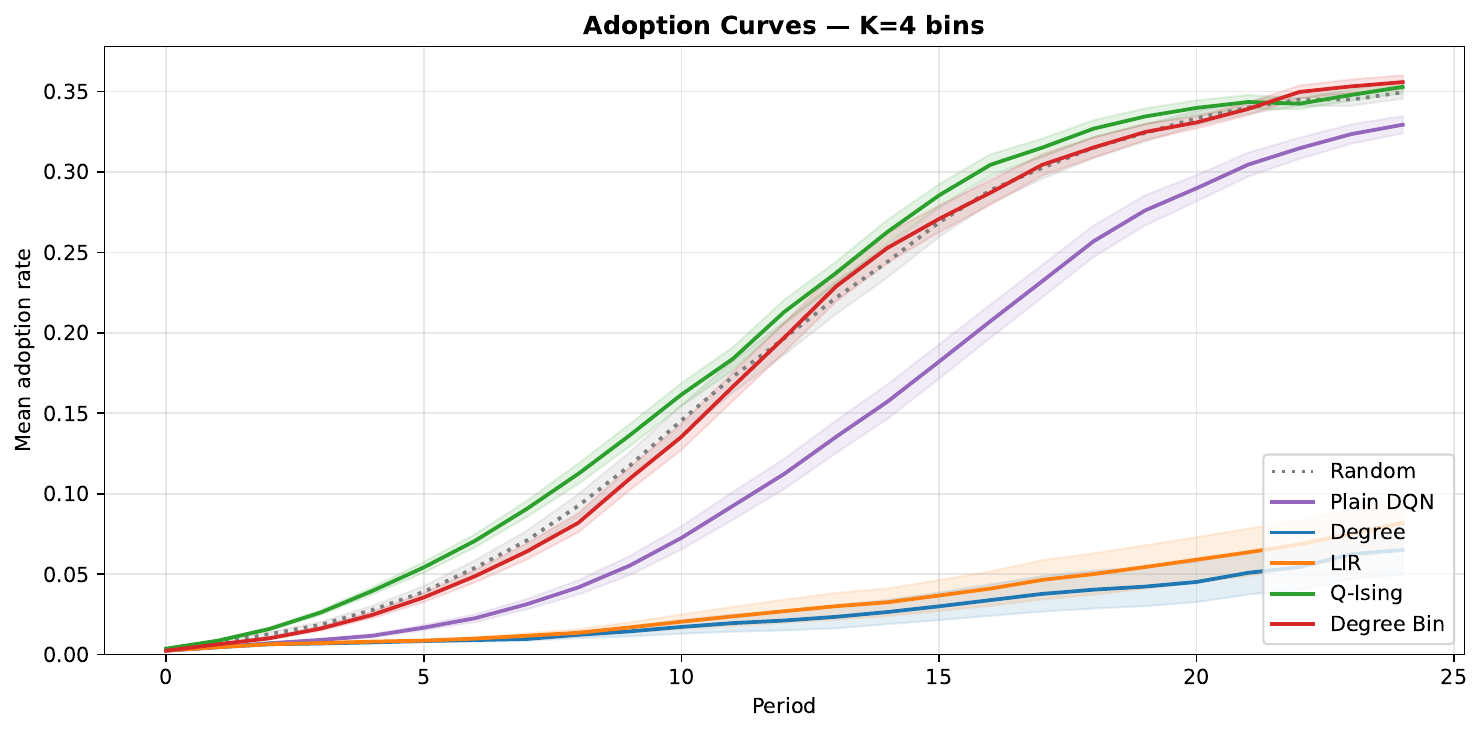}
    \includegraphics[width=0.45\linewidth]{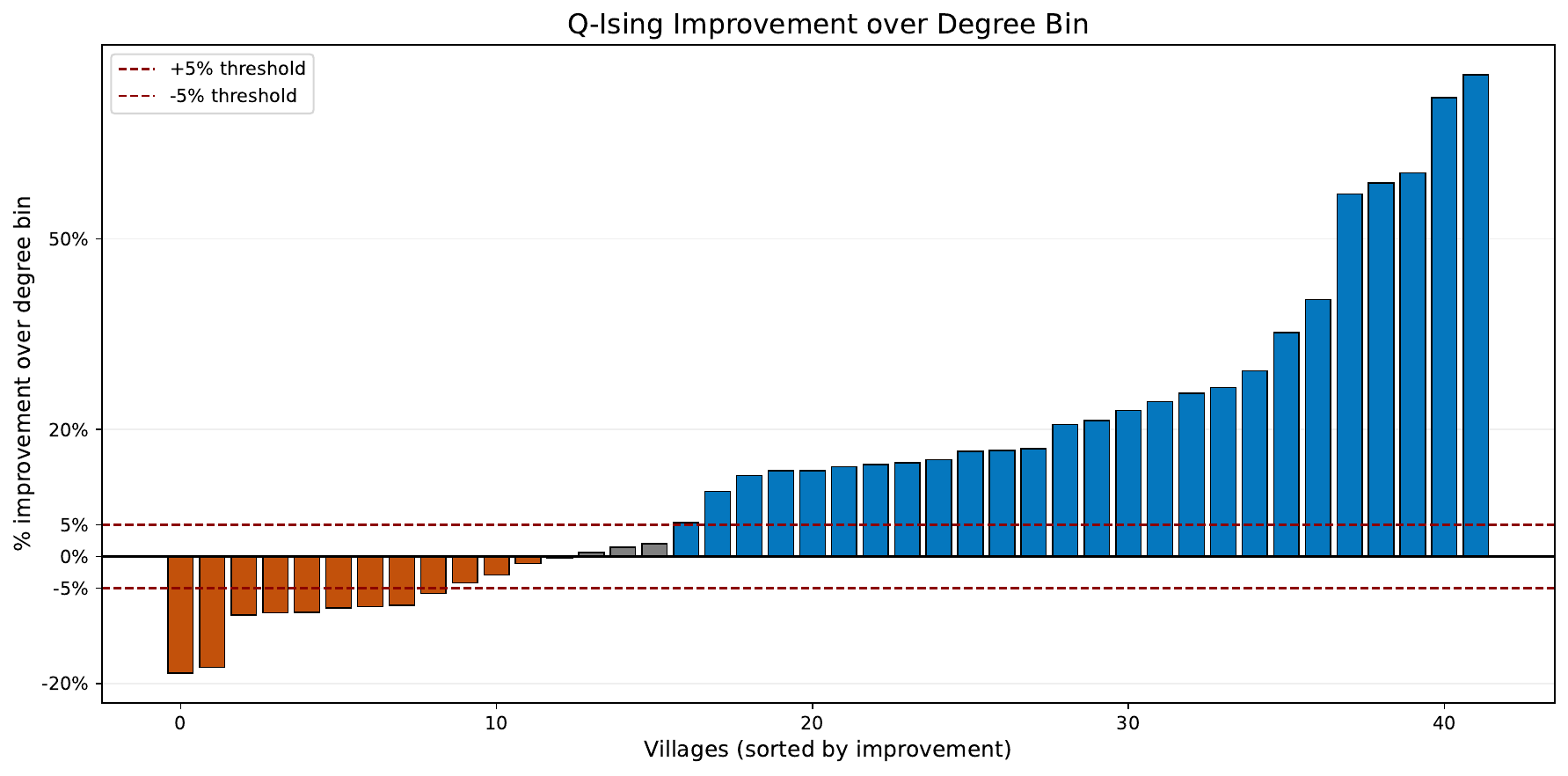}
    \caption{On the left: The mean period reward from different policies over time averaged over 50 test episodes with standard deviations shaded on the SBM data. On the right: The percentage improvement of Q-Ising over the best topological alternative (degree-bin) in microfinance Villages over 50 test runs.}
    \label{fig:improvement}
\end{figure}

In this set-up, random policy outperforms most of the topological heuristics because it seeds the highly-susceptible blocks by chance occasionally. So, it matches
the right target more often than a method deterministically committed to
the wrong one. Plain DQN learns a qualitatively similar seed-priority ordering to Q-Ising but has a slow initial ramp. This suggests that Ising augmentation is most valuable in the early campaign window, in addition to providing posterior ensemble policies, interpretable parameters and uncertainty quantification.

\subsection{Experiments in Indian Microfinance Villages}
\label{subsec:villages}

We use the empirical adjacency matrices of Indian microfinance villages to provide realistic clustering and degree distributions while simulating SIS dynamics to maintain control over group heterogeneity. To define the bins, we use the edge-betweenness based clustering algorithm. We only consider villages with more than one cluster, 42 of 43 villages satisfy this condition. If an identified cluster has less than 10 nodes, these nodes are considered to be the part of the largest cluster. Each bin gets assigned an unobserved spread and churn rate. The planner observes a single offline panel of $T_{\mathrm{train}} = 500$ periods collected under a uniform random bin policy for each of these networks. The details of the experimental set-up can be found in Appendix \ref{app:experiment}. Evaluation uses $50$ independent test runs over test horizon $H = 25$ initialized from no adoption.

Table ~\ref{tab:village_results} in Appendix~\ref{app:tables} reports the results. A descriptive figure of performance improvement of Q-Ising over the best non-adaptive policy which is Degree-bin can be found in Figure ~\ref{fig:improvement}. Q-Ising effectively learns to concentrates its initial treatments on communities with high spread rates and communities that are well connected with others, generating compounding organic spread. Then, it adaptively treats other communities depending on the activity levels. Its performance improvement is more pronounced in villages with low modularity because Q-Ising relies on identifying influential communities and adaptively targeting the remaining ones. In highly modular networks, spillovers compound less, diminishing the advantage of Q-Ising. 

\begin{figure}
    \centering
    \includegraphics[width=1\linewidth]{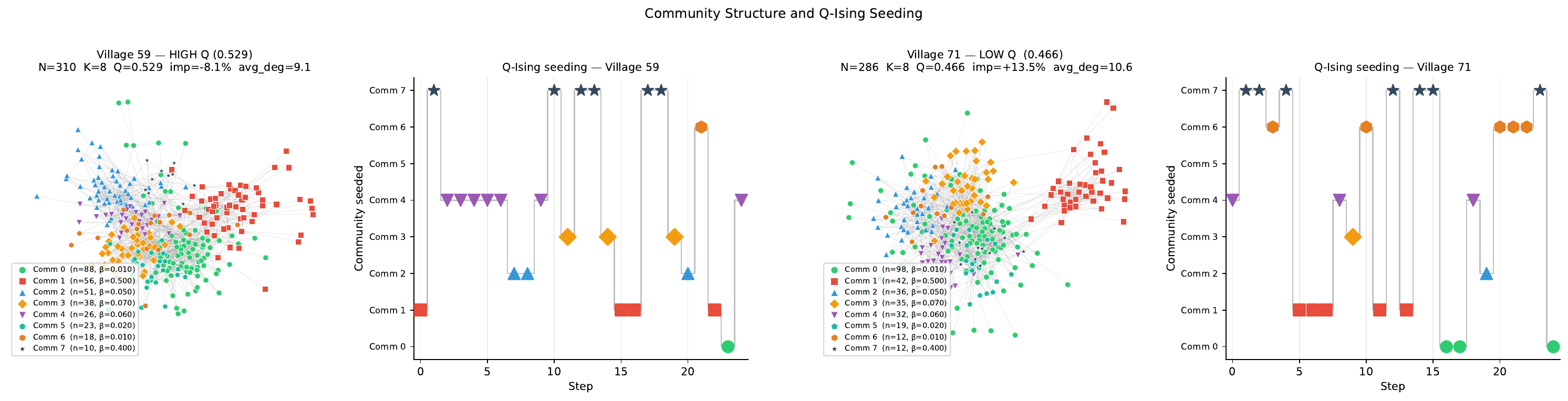}
    \caption{Trajectory differences across two different networks. The village on the right has a more modular community structure, so identified communities do not spread to other communities as much as the village on the left. 
     The correlation between Q-Ising's relative performance improvement over Degree-Bin with the village's modularity is approximately -0.5.}
    \label{fig:structure_trajectory}
\end{figure}

In addition to the performance improvement, Q-Ising reveals important features about the underlying mechanism through its estimated coefficients. The intercepts $\beta_0$ are negative for all bins, reflecting a strong baseline non-adoption tendency. Treatment effects $\beta_1$ are strongly positive for all the bins, recovering the "perfect treatment" structure of the SIS model. Persistence coefficients $\beta_2$ are positive, capturing adoption stickiness; communities with higher churn rates show weaker persistence. The dominant estimated couplings originate from communities with high spread rates. This interpretable analysis is not possible with other black-box policy methods. Figures about Ising fit quality and posterior distributions can be found in Appendix ~\ref{app:figures}.

With the ensemble policy approach, Q-Ising can also learn the policy uncertainty. The parameter estimation takes approximately 1 minute of wall clock time with 500 iterations. The training of the 20 agents corresponding to 20 posterior draws takes around 20 minutes. In general, the majority-vote path starts with near-unanimous agreement on communities with high spread rate. After the spreader communities has largely saturated, the majority vote shifts toward multiple alternatives. This dispersion identifies periods
near a critical threshold and serves as a measure of strategic uncertainty that a point-estimate policy would silently suppress. Figures showing the ensemble policy trajectory for an example village can be found in Appendix ~\ref{app:figures}.

\section{Conclusion}
\label{sec:conclusion}

This paper develops  Q-Ising, a framework for dynamic treatment allocation under network interference from observational panel data. This is one of the first attempts to offline dynamic policy learning in networks despite its relevance for public health, microfinance, and other settings where experimentation is costly or infeasible. Our approach combines a Bayesian Ising model of network dynamics with conservative Q-learning, yielding a policy that comes with structural parameter estimates and posterior uncertainty while remaining competitive with model-free offline RL.

There are many refinements possible. The framework currently seeds a single node per period; multi-node budget extensions are open directions. The method also requires a sufficiently long observational panel for adequate state-action coverage, and its behavior under severe distributional shift between the behavior policy and the target policy requires further study. For state augmentation with network information, higher order Ising interactions or other network representations such as graph neural networks can be used when there is sufficient data.  Our experiments demonstrate the framework on SIS dynamics, but the pipeline accommodates to other contagion models that use synchronous updating.
Positive applications include public health interventions, microfinance outreach, and information campaigns in low-resource settings. The same framework could be used to exploit social influence for commercial or political ends, and we encourage practitioners to reflect on deployment context accordingly.

\bibliographystyle{plainnat}
\bibliography{proposal}      

\appendix
\section{Tables}
\label{app:tables}
\begin{table}[H]
\centering
\caption{Comparison of related methods. The check mark the property is satisfied, cross indicates it is not, and -- indicates not applicable.}
\label{tab:comparison}
\small
\begin{tabular}{lcccc}
\toprule
\textbf{Method} & \textbf{Network} & \textbf{Dynamic} & \textbf{Data is} & \textbf{Uncertainty} \\
                & \textbf{interference} & \textbf{policy} & \textbf{observational} & \textbf{quantification} \\
\midrule
\multicolumn{5}{l}{\textit{Empirical welfare maximization}} \\[2pt]
\citet{kitagawa2018empirical}        & \xmark & \xmark & \cmark & \xmark \\
\citet{viviano2025policy}            & \cmark & \xmark & \cmark & \xmark \\
\midrule
\multicolumn{5}{l}{\textit{Topological policies}} \\[2pt]
\citet{kempe2003maximizing}          & \cmark & \xmark & -- & \xmark \\
\citet{banerjee2013diffusion}           & \cmark & \xmark & -- & \xmark \\
\citet{chen2009efficient}                & \cmark & \xmark & -- & \xmark \\
\citet{liu2017fast}           & \cmark & \xmark & -- & \xmark \\
\midrule
\multicolumn{5}{l}{\textit{Dynamic treatment regimes}} \\[2pt]
\citet{chakraborty2014dynamic}           & \xmark & \cmark & \cmark & \xmark \\
\citet{hu2025optimal}                & \xmark & \cmark & \cmark & \xmark \\
\citet{kitagawa2022policy}           & \xmark & \cmark & \cmark & \xmark \\
\midrule
\multicolumn{5}{l}{\textit{Network bandits}} \\[2pt]
\citet{herlihy2023networked}         & \cmark & \cmark & \xmark & \xmark \\
\citet{vaswani2015influence}         & \cmark & \cmark & \xmark & \xmark \\
\midrule
\multicolumn{5}{l}{\textit{GNN / simulation-based}} \\[2pt]
\citet{manchanda2020gcomb}           & \cmark & \xmark & \xmark & \xmark \\
\citet{meirom2021controlling}        & \cmark & \cmark & \xmark & \xmark \\
\midrule
\textbf{Q-Ising}              & \cmark & \cmark & \cmark & \cmark \\
\bottomrule
\end{tabular}
\end{table}

\setlength{\tabcolsep}{4pt}
\begingroup
\small
\setlength{\tabcolsep}{3.5pt}
\begin{longtable}{lcccccc}
\caption{Mean reward (mean $\pm$ std) for each seeding policy across villages. Bolding indicates policies within 1 std of the best performer. All values rounded to three decimal places.} \\
\label{tab:village_results} \\
\toprule
Village & Random & Degree & LIR & Degree-Bin & Plain DQN & Q-Ising \\
\midrule
\endfirsthead
\multicolumn{7}{c}{\tablename\ \thetable{} (continued)} \\
\toprule
Village & random & Degree & LIR & Degree-Bin & Plain DQN & Q-Ising \\
\midrule
\endhead
\midrule
\multicolumn{7}{r}{Continued on next page} \\
\endfoot
\bottomrule
\endlastfoot
0 & 0.040 (0.002) & 0.069 (0.002) & 0.069 (0.002) & \textbf{0.074 (0.002)} & 0.057 (0.002) & 0.061 (0.002) \\
1 & 0.015 (0.001) & 0.025 (0.001) & 0.023 (0.001) & \textbf{0.026 (0.001)} & \textbf{0.025 (0.001)} & 0.022 (0.001) \\
2 & 0.050 (0.003) & 0.066 (0.003) & 0.062 (0.003) & 0.088 (0.002) & \textbf{0.096 (0.003)} & 0.080 (0.003) \\
3 & 0.018 (0.001) & \textbf{0.028 (0.001)} & \textbf{0.027 (0.001)} & \textbf{0.028 (0.001)} & \textbf{0.028 (0.001)} & 0.026 (0.001) \\
4 & 0.042 (0.002) & 0.062 (0.002) & 0.061 (0.002) & \textbf{0.069 (0.002)} & 0.065 (0.002) & 0.063 (0.002) \\
5 & 0.020 (0.001) & 0.024 (0.001) & 0.024 (0.001) & \textbf{0.033 (0.001)} & 0.029 (0.001) & 0.031 (0.001) \\
6 & 0.018 (0.001) & \textbf{0.035 (0.001)} & \textbf{0.036 (0.001)} & 0.031 (0.001) & 0.029 (0.001) & 0.029 (0.001) \\
7 & 0.052 (0.003) & \textbf{0.096 (0.003)} & \textbf{0.094 (0.002)} & 0.084 (0.003) & 0.086 (0.003) & 0.077 (0.003) \\
8 & 0.017 (0.001) & \textbf{0.046 (0.001)} & \textbf{0.046 (0.001)} & 0.037 (0.001) & 0.038 (0.001) & 0.035 (0.001) \\
9 & 0.028 (0.001) & 0.028 (0.001) & 0.028 (0.001) & 0.044 (0.001) & \textbf{0.050 (0.001)} & 0.042 (0.001) \\
10 & 0.037 (0.002) & 0.053 (0.002) & 0.053 (0.002) & \textbf{0.069 (0.002)} & \textbf{0.066 (0.002)} & \textbf{0.067 (0.003)} \\
11 & 0.029 (0.002) & 0.031 (0.001) & 0.031 (0.001) & \textbf{0.054 (0.002)} & \textbf{0.055 (0.002)} & \textbf{0.054 (0.002)} \\
12 & 0.023 (0.001) & 0.031 (0.001) & 0.030 (0.001) & \textbf{0.039 (0.001)} & \textbf{0.040 (0.001)} & \textbf{0.039 (0.001)} \\
13 & 0.023 (0.001) & 0.026 (0.001) & 0.026 (0.001) & \textbf{0.044 (0.001)} & \textbf{0.043 (0.002)} & \textbf{0.044 (0.002)} \\
14 & 0.030 (0.002) & 0.035 (0.001) & 0.040 (0.002) & \textbf{0.051 (0.002)} & \textbf{0.055 (0.001)} & \textbf{0.052 (0.002)} \\
15 & 0.038 (0.002) & 0.053 (0.002) & 0.053 (0.002) & 0.060 (0.002) & \textbf{0.065 (0.002)} & 0.061 (0.002) \\
16 & 0.042 (0.003) & \textbf{0.107 (0.003)} & \textbf{0.107 (0.003)} & 0.084 (0.003) & 0.091 (0.003) & 0.089 (0.003) \\
17 & 0.043 (0.002) & 0.033 (0.001) & 0.037 (0.001) & 0.061 (0.001) & \textbf{0.067 (0.002)} & \textbf{0.067 (0.001)} \\
18 & 0.023 (0.001) & 0.029 (0.001) & 0.029 (0.001) & 0.040 (0.001) & \textbf{0.044 (0.001)} & \textbf{0.045 (0.001)} \\
19 & 0.027 (0.001) & 0.033 (0.001) & 0.033 (0.001) & 0.041 (0.001) & \textbf{0.047 (0.001)} & \textbf{0.046 (0.001)} \\
20 & 0.055 (0.004) & 0.075 (0.003) & 0.072 (0.004) & 0.082 (0.003) & \textbf{0.097 (0.003)} & \textbf{0.093 (0.003)} \\
21 & 0.023 (0.001) & 0.042 (0.001) & \textbf{0.044 (0.001)} & 0.038 (0.001) & \textbf{0.046 (0.002)} & \textbf{0.044 (0.002)} \\
22 & 0.026 (0.001) & \textbf{0.052 (0.002)} & \textbf{0.052 (0.001)} & 0.044 (0.002) & \textbf{0.050 (0.002)} & \textbf{0.050 (0.002)} \\
23 & 0.051 (0.002) & 0.061 (0.002) & 0.061 (0.002) & 0.068 (0.002) & \textbf{0.085 (0.002)} & 0.078 (0.002) \\
24 & 0.037 (0.002) & 0.056 (0.002) & 0.056 (0.002) & 0.056 (0.002) & \textbf{0.066 (0.002)} & \textbf{0.064 (0.002)} \\
25 & 0.040 (0.003) & 0.029 (0.001) & 0.025 (0.001) & 0.056 (0.002) & \textbf{0.065 (0.002)} & \textbf{0.065 (0.002)} \\
26 & 0.021 (0.001) & 0.037 (0.001) & 0.039 (0.001) & 0.036 (0.001) & \textbf{0.041 (0.001)} & \textbf{0.042 (0.001)} \\
27 & 0.025 (0.001) & 0.035 (0.001) & 0.034 (0.001) & 0.036 (0.001) & \textbf{0.042 (0.001)} & \textbf{0.042 (0.001)} \\
28 & 0.054 (0.003) & 0.076 (0.002) & 0.081 (0.002) & 0.082 (0.002) & \textbf{0.103 (0.003)} & \textbf{0.099 (0.002)} \\
29 & 0.035 (0.002) & 0.053 (0.001) & 0.053 (0.001) & 0.049 (0.002) & \textbf{0.060 (0.002)} & \textbf{0.060 (0.001)} \\
30 & 0.043 (0.002) & 0.049 (0.001) & 0.049 (0.001) & 0.063 (0.002) & \textbf{0.078 (0.002)} & \textbf{0.078 (0.001)} \\
31 & 0.035 (0.002) & \textbf{0.078 (0.002)} & \textbf{0.078 (0.002)} & 0.054 (0.002) & 0.061 (0.002) & 0.067 (0.002) \\
32 & 0.034 (0.001) & 0.041 (0.001) & 0.046 (0.002) & 0.055 (0.002) & \textbf{0.068 (0.001)} & \textbf{0.069 (0.002)} \\
33 & 0.057 (0.003) & 0.066 (0.002) & 0.066 (0.002) & 0.078 (0.002) & \textbf{0.108 (0.002)} & 0.098 (0.003) \\
34 & 0.040 (0.002) & 0.059 (0.002) & 0.056 (0.002) & 0.055 (0.002) & \textbf{0.068 (0.002)} & \textbf{0.071 (0.002)} \\
35 & 0.045 (0.002) & 0.050 (0.001) & 0.050 (0.001) & 0.063 (0.001) & \textbf{0.090 (0.002)} & 0.086 (0.002) \\
36 & 0.046 (0.002) & 0.038 (0.001) & 0.038 (0.001) & 0.063 (0.002) & \textbf{0.088 (0.002)} & \textbf{0.088 (0.002)} \\
37 & 0.069 (0.003) & 0.082 (0.002) & 0.083 (0.002) & 0.079 (0.002) & \textbf{0.128 (0.002)} & \textbf{0.124 (0.002)} \\
38 & 0.035 (0.001) & 0.035 (0.001) & 0.035 (0.001) & 0.046 (0.001) & \textbf{0.071 (0.002)} & \textbf{0.073 (0.002)} \\
39 & 0.052 (0.002) & 0.068 (0.002) & 0.068 (0.002) & 0.067 (0.002) & 0.103 (0.002) & \textbf{0.107 (0.002)} \\
40 & 0.052 (0.002) & 0.021 (0.001) & 0.021 (0.001) & 0.047 (0.001) & \textbf{0.079 (0.001)} & \textbf{0.080 (0.001)} \\
41 & 0.023 (0.001) & 0.012 (0.000) & 0.012 (0.000) & 0.026 (0.001) & 0.045 (0.001) & \textbf{0.046 (0.001)} \\
\end{longtable}

\endgroup

\section{Additional Figures}
\label{app:figures}

\begin{figure}[H]
    \centering
    \includegraphics[width=0.8\linewidth]{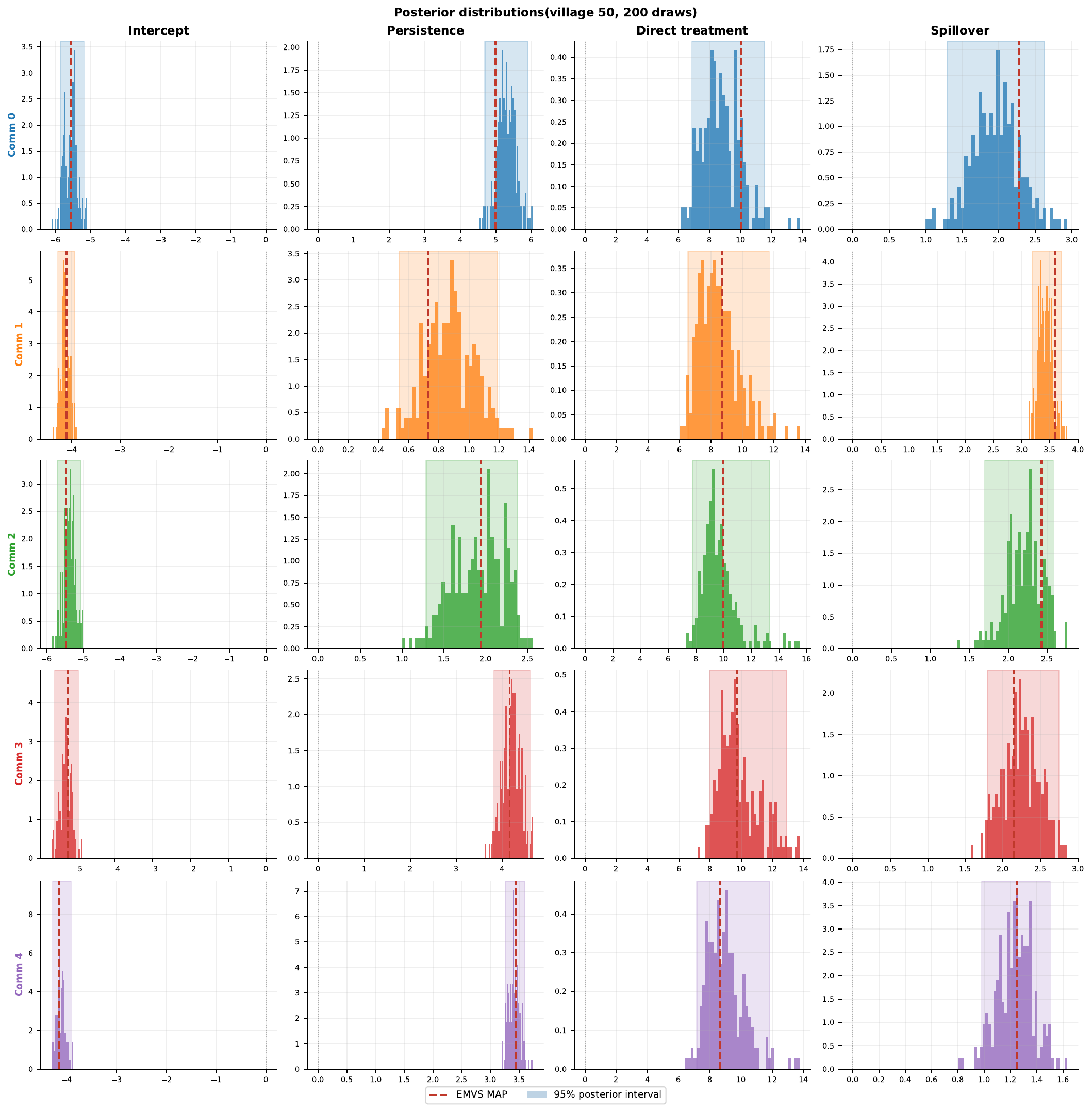}
    \includegraphics[width=0.8\linewidth]{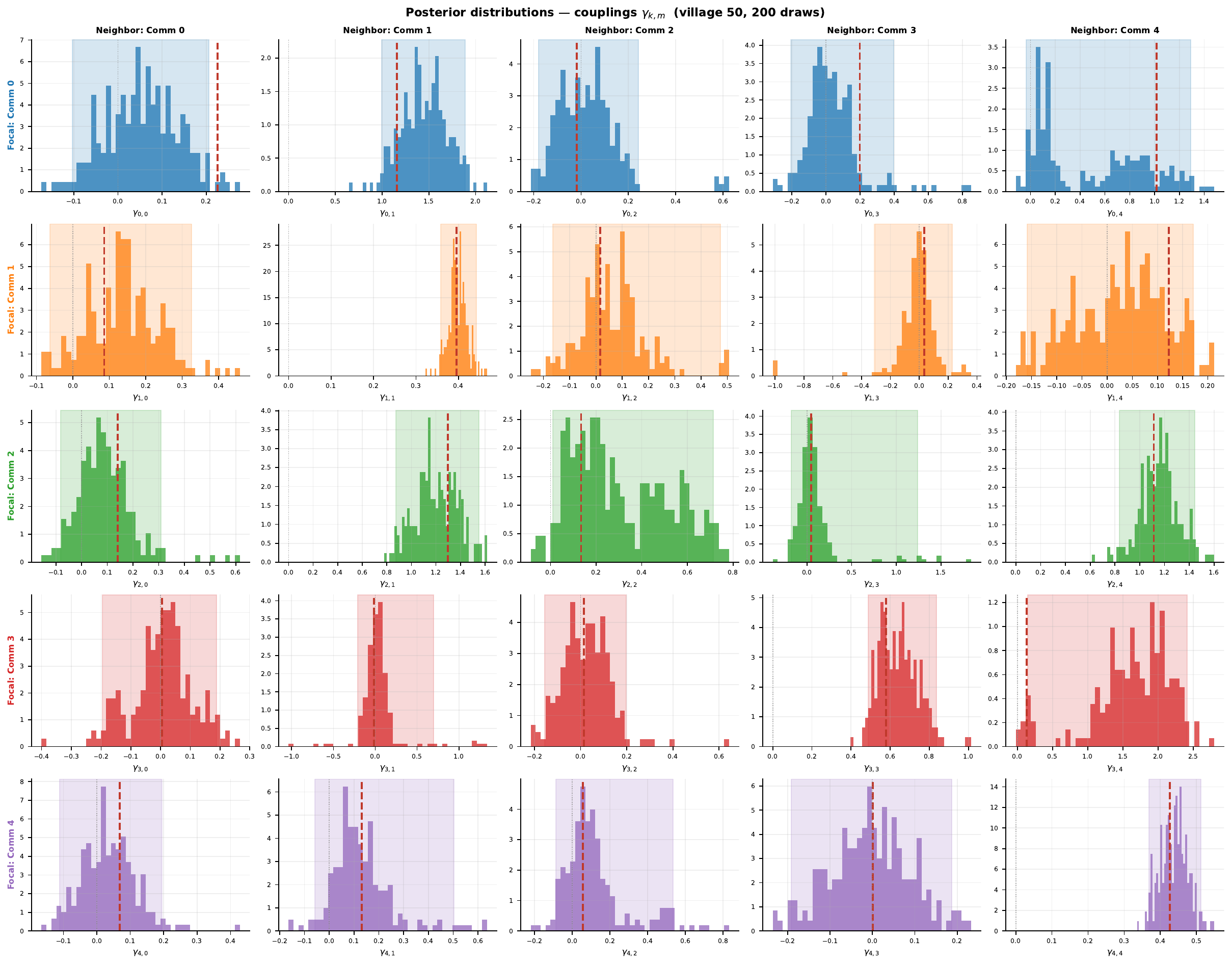}
    \caption{The posterior distribution of dynamic Ising parameters estimated by MCMC for Village 50.}
    \label{fig:posterior}
\end{figure}

\begin{figure}[H]
    \centering
    \includegraphics[width=1\linewidth]{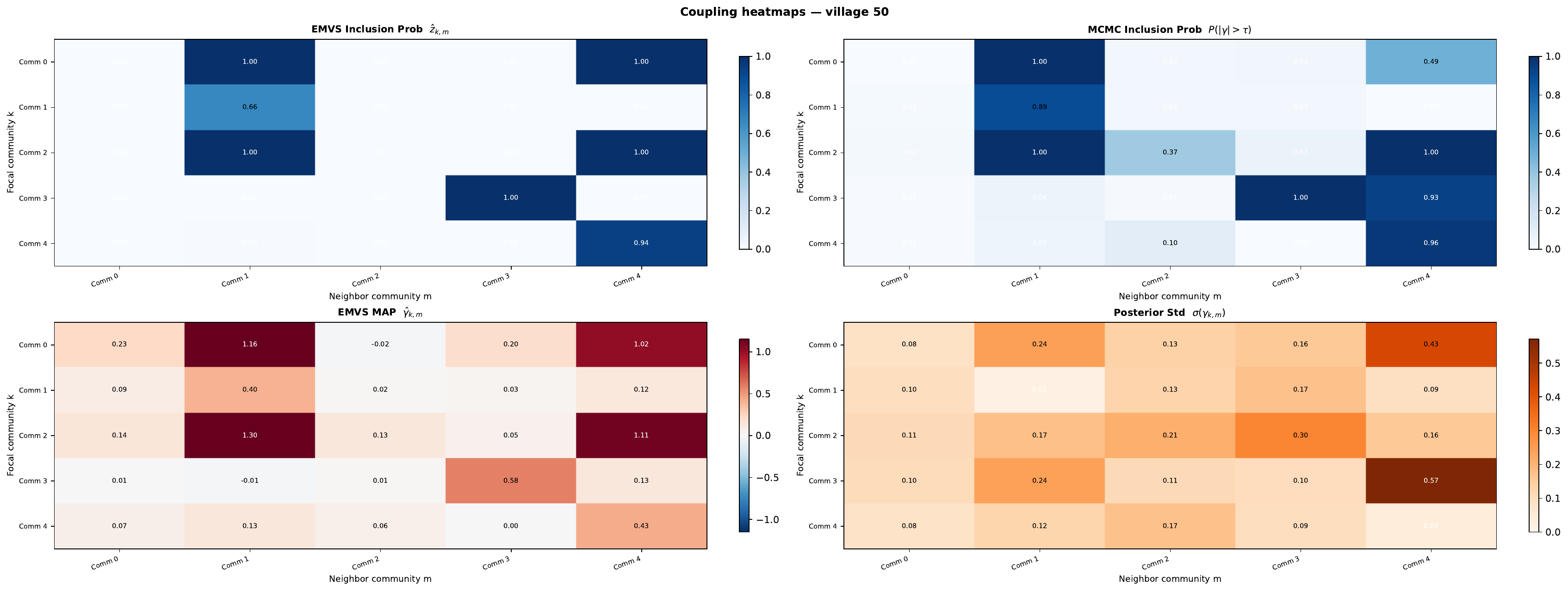}
    \caption{Estimated inclusion probabilities for coupling parameters for Village 50 by EMVS and MCMC, MAP estimates for the couplings and posterior standard deviations.}
    \label{fig:coupling}
\end{figure}

\begin{figure}[H]
    \centering
    \includegraphics[width=0.75\linewidth]{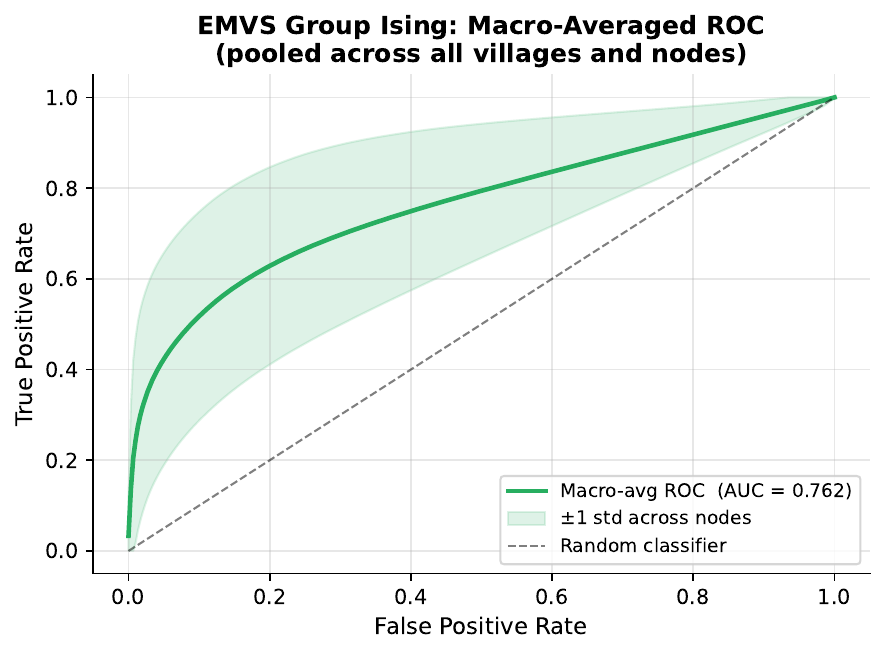}
    \caption{The AUC of pooled nodes for the microfinance villages.}
    \label{fig:auc}
\end{figure}

\begin{figure}[H]
    \centering
    \includegraphics[width=1\linewidth]{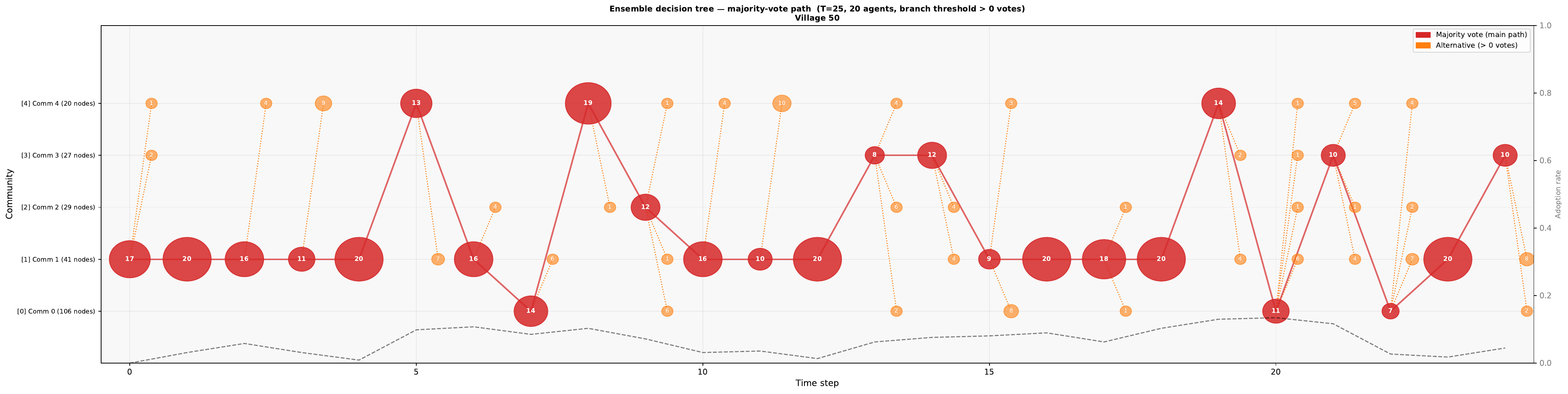}
    \caption{MCMC ensemble majority-vote path for Village 50.
      Bubble area encodes the number of agents voting for each action;
      orange bubbles mark alternatives receiving more than one vote.
      Concentrated bubbles in early periods reflect high posterior
      agreement on treating certain communities, dispersed votes later on signal near-equivalence of strategies once spreader communities saturated.}
    \label{fig:ensemble}
\end{figure}

\section{Proofs}
\label{app:proofs}

\subsection{Proof of Greedy Suboptimality (Proposition \ref{prop:heur_sub}) }
\label{app:greedy_sub}

\begin{proof}
    We provide a counterexample under basic SIS dynamics. Consider a graph $\mathcal{G}$ with nodes $\{A, B, C\}$. Nodes $A$ and $B$ are connected while $C$ is isolated. The connected nodes $\{A, B\}$ has spread rate $\rho > 0$ and a churn rate of $1$, meaning that they stay adopted only for one period, while node $C$ has a churn rate of $0$, meaning it stays adopted once treated. Let's assume the planner wants to maximize total adoption over $T=2$ periods. A greedy policy would treat one of the connected nodes because the expected immediate reward for treating $A$ is higher than treating $C$. $\mathbb{E}[R_{A,1}] = 1 + \rho > \mathbb{E}[R_{C,1}] = 1$.  However, at $t=2$ both of these nodes would churn unless they are re-seeded. The optimal policy treats $C$ first which does not require any maintenance in future periods, and then treats any one of the connected nodes resulting in $\mathbb{E}[R_{C,1} + R_{A,2}] = 1 + 1 + 1 + \rho =  3 + \rho$. The greedy policy can achieve $\mathbb{E}[R_{A,1} + R_{B,2}] = 2 + 2\rho$. In this example, the greedy policy is suboptimal. 
\end{proof}

\subsection{Proof of Dynamic Ising Remark}
\label{app:ising_proofs}

\begin{remark}[Synchronous Updates Are Not Element-wise Representable]
\label{prop:sync_not_ising}
Consider the dynamic Ising model on $N \geq 4$ nodes under no treatment 
($\beta_{2,i} = 0$) and symmetric interactions $\gamma_{ij} = \gamma_{ji}$. 
Under synchronous updating, the stationary distribution cannot in general 
be expressed as an element-wise Gibbs measure 
$\mu(y) \propto \exp\big(\sum_i h_i y_i + \sum_{i<j} \beta_{ij} y_i y_j\big)$.
\end{remark}

\begin{proof}
Since $\sigma$ maps $\mathbb{R}$ to $(0,1)$, every transition probability is 
strictly positive, so the chain admits a unique stationary distribution $\mu$. We give a counterexample to show that this stationary distribution cannot be in general represented as an element-wise Gibbs measure. 

Take $N=4$ on the complete graph with $\beta_{0,i}=0$ and $\gamma_{ij}=1$ 
for all $i\neq j$. Under synchronous updates, all nodes flip independently 
given the current state $y$, so the transition kernel factorizes:
\[
K(y \to y') = \prod_{i=1}^{4} \sigma\!\Big(\textstyle\sum_{j\neq i} y_j\Big)^{y'_i}
\Big(1 - \sigma\!\big(\textstyle\sum_{j\neq i} y_j\big)\Big)^{1-y'_i}
\]
where each factor is the probability that node $i$ updates to $y'_i$ given 
that it sees field $\sum_{j\neq i} y_j$. This defines the $16\times 16$ 
matrix $K$ on $\{0,1\}^4$, and $\mu$ is its unique left eigenvector with 
eigenvalue $1$, normalized to a probability distribution.

By permutation symmetry, $\mu(y)$ depends only on the weight $m = \sum_i y_i$; 
write $\mu_m$ for this common value. Any element-wise Gibbs measure 
reproducing $\mu$ must share this symmetry, forcing uniform field $h$ and 
uniform coupling $\beta$, so that 
$\mu_m \propto \exp\big(hm + \beta\binom{m}{2}\big)$. Taking log-ratios,
\[
\Delta_m := \log\frac{\mu_{m+1}}{\mu_m} 
= h + \beta\Big[\tbinom{m+1}{2} - \tbinom{m}{2}\Big] 
= h + \beta m,
\]
which is affine in $m$. Solving $\boldsymbol{\mu}K = \boldsymbol{\mu}$ 
numerically yields
\[
\Delta_0 \approx 1.860,\quad \Delta_1 \approx 2.247,\quad 
\Delta_2 \approx 2.549,\quad \Delta_3 \approx 2.765,
\]
with successive differences $0.387,\ 0.302,\ 0.216$. These are not constant, 
so $\Delta_m$ is not affine in $m$, and no element-wise Gibbs measure can 
reproduce $\mu$.

For $N > 4$, embed the four-node complete graph and set all remaining 
couplings to zero. The disconnected nodes evolve independently of the 
embedded clique, so the marginal of $\mu$ on the clique coincides with the 
four-node stationary distribution above. If $\mu$ were element-wise Gibbs 
on $N$ nodes, the Hamiltonian would split, and marginalizing out the 
disconnected nodes would yield an element-wise Gibbs measure on the 
clique.
\end{proof}

\section{Regret Analysis} \label{sec:ra-proof}

\subsection{Notations Setup} \label{subsec:proof-notation}

We review and define additional notations that are useful for regret analysis.  The main text states the guarantee for a compact PEVI analogue of Q-Ising. Here we make the PEVI recursion and the aggregate processes explicit.

\paragraph{Oracle and estimated Q-Ising states}
Let $Z_{\tau,h}$ denote the pre-action full-network history at stage $h$
of effective training block $\tau$. The oracle and estimated Q-Ising states
are $s_{\tau,h}^\star = \phi_{\theta_0}(Z_{\tau,h}),$ and $\widehat s_{\tau,h}
=\phi_{\widehat\theta}(Z_{\tau,h}).$ For the concentration arguments below and in the lemma section, $\widehat s_{\tau,h}$ should be read as a cross-fitted state, and we suppress the fold index in the notation. This is a proof device. In practice, the empirical implementation may use the full first-stage fit.

For any state-based policy $\pi=(\pi_1,\ldots,\pi_H)$, with
$\pi_h:\mathcal S\to[K]$, its oracle and plug-in full-history deployments are
$\pi_h^\circ(z)=\pi_h(\phi_{\theta_0}(z))$ and $\pi_h^{\mathrm{plug}}(z)=\pi_h(\phi_{\widehat\theta}(z)).$
The theorem in the main text analyzes oracle deployment. 

\paragraph{Effective sample size}
For regret analysis, we represent the offline panel as stage-wise transition sets
$$
\mathcal D_h=\{(\widehat s_{\tau,h},b_{\tau,h},r_{\tau,h},
      \widehat s_{\tau,h+1})\}_{\tau\in\mathcal I_h},
  \qquad h=1,\ldots,H,
$$
where $\tau$ indexes effective training blocks from the training panel. Let $n_h:=|\mathcal I_h|$ denote the effective number of approximately independent transitions at stage
$h$, after accounting for temporal dependence. We write 
\begin{equation*}
    n_{\mathrm{eff}}:=\min_{h\in[H]} n_h,
  \qquad
  n_{\log}:=\max_{h\in[H]} n_h .
\end{equation*}
The minimum $n_{\mathrm{eff}}$ controls simplified coverage rates, whereas $n_{\log}$ appears only inside the logarithmic confidence radius in the
self-normalized regression bound. Under balanced blocking,
$n_h=n_{\mathrm{eff}}=n_{\log}$ for all $h$. 

\paragraph{PEVI Regressions Review \citep{jin2021pessimism}} Recall that $\varphi(\widehat s_{\tau,h},b_{\tau,h})$ is a normalized feature map with $\| \varphi(s,b)\|_2\le 1$.  let $[x]_{[0,c]}=\min\{\max\{x,0\},c\}$. Setting
$\widehat V_{H+1}=0$, PEVI computes, for $h=H,\ldots,1$, in a backward fashion
\begin{align*}
&\Lambda_h = \lambda I_{d_{\phi}}+ \sum_{\tau\in\mathcal I_h} \varphi(\widehat s_{\tau,h},b_{\tau,h})
  \varphi(\widehat s_{\tau,h},b_{\tau,h})^\top, \quad  \widehat w_h = \Lambda_h^{-1} \sum_{\tau\in\mathcal I_h} \varphi(\widehat s_{\tau,h},b_{\tau,h}) \{r_{\tau,h}+\widehat V_{h+1}(\widehat s_{\tau,h+1})\}  \\
  &\widehat Q_h(s,b) = \left[ \varphi(s,b)^\top\widehat w_h - \beta \sqrt{\varphi(s,b)^\top\Lambda_h^{-1}\varphi(s,b)} \right]_{[0,H-h+1]}, \qquad \widehat V_h(s)=\max_b \widehat Q_h(s,b).
\end{align*}
The learned PEVI policy is $ \widehat\pi_h(s)\in\arg\max_{b\in[K]}\widehat Q_h(s,b)$. Its oracle and plug-in full-history deployments are $ \widehat\pi_h^\circ(z)= \widehat\pi_h(\phi_{\theta_0}(z))$ and $\widehat\pi_h^{\mathrm{plug}}(z) = \widehat\pi_h(\phi_{\widehat\theta}(z))$ respectively.

\paragraph{Aggregate MDP and value-to-go} 
Let $\mathcal M^S$ be the target oracle aggregate MDP over Q-Ising states, where a chosen bin is implemented by uniform within-bin randomization. Let $\mathcal M^\mu$ be the logged oracle aggregate MDP induced by the historical within-bin selector after node actions are mapped to bins. For $\mathcal M\in\{\mathcal M^S,\mathcal M^\mu\}$, write
$r_h^{\mathcal M}$ and $P_h^{\mathcal M}$ for the reward and transition kernel, and define
$$ (\mathcal T_h^{\mathcal M}V)(s,b)=
  r_h^{\mathcal M}(s,b)+ \mathbb E_{s'\sim P_h^{\mathcal M}(\cdot\mid s,b)}[V(s')].$$
For $\mathcal M\in\{\mathcal M^S,\mathcal M^\mu\}$ and
$\pi\in\Pi_S$, define the aggregate value-to-go by
$$V_h^{\mathcal M,\pi}(s) := \mathbb E_{\mathcal M}^{\pi} \left[ \sum_{u=h}^H r_u^{\mathcal M}(s_u,b_u) \,\middle|\, s_h=s \right],\qquad
V_{H+1}^{\mathcal M,\pi}=0.$$
When $\mathcal M=\mathcal M^S$, we write $V_h^{S,\pi}$; when $\mathcal M=\mathcal M^\mu$, we write $V_h^{\mu,\pi}$.

Finally, we note that the PEVI uncertainty of a policy $\pi$ under the logged aggregate MDP is
$$\mathfrak U_{\mathcal D}^{\mu}(\pi)
=
\sum_{h=1}^{H}
\mathbb E_{\mathcal M^\mu}^{\pi}
\left[
\sqrt{
\varphi(s_h,b_h)^\top
\Lambda_h^{-1}
\varphi(s_h,b_h)}
\right].$$

\subsection{Regularity Assumptions} \label{subsec:ra-assumptions}

We state the regularity conditions used in the regret proof. In words, Assumption~\ref{ass:valid_offline_compact} handles
causal validity and the effective blocked sample; Assumption~\ref{ass:qising_abstraction_compact}
handles network-to-state abstraction and logged-versus-target within-bin
selection; Assumption~\ref{ass:first_stage_compact} handles first-stage
Q-Ising estimation; and Assumption~\ref{ass:linear_bellman_compact} handles
approximate linear PEVI regression.

For each stage $h$ and each value function $V$ appearing in the PEVI
backward recursion, define the regression target
$$Y_{\tau,h}(V) := r_{\tau,h}+V(\widehat s_{\tau,h+1}).$$
The assumptions below imply that $Y_{\tau,h}(V)$ admits the decomposition
\begin{equation}
Y_{\tau,h}(V) =
x_{\tau,h}^\top w_h(V) + \xi_{\tau,h}(V)
+ e_{\tau,h}^{\mathrm{lin}}(V) + e_{\tau,h}^{\theta}(V),
\label{eq:regression_residual_decomposition}
\end{equation}
where $\xi_{\tau,h}(V)$ is the stochastic Bellman noise, $e_{\tau,h}^{\mathrm{lin}}(V)$ is the Bellman linear-approximation residual,
and $e_{\tau,h}^{\theta}(V)$ is the additional residual induced by using
estimated Q-Ising states instead of oracle Q-Ising states.

\begin{assumption}[Causal validity and effective PEVI Samples]
\label{ass:valid_offline_compact}
Logged actions are valid interventions. For each stage $h$,
\[
a_h\perp \{(Z_{h+1}(a),r_h(a)):a\in[N]\}\mid Z_h,
\qquad
(Z_{h+1},r_h)=(Z_{h+1}(a_h),r_h(a_h)).
\]
The training panel can be represented by effective stage-wise
blocks $\{\mathcal I_h\}_{h=1}^H$. For each stage $h$, after conditioning
on the data used to construct the cross-fitted Q-Ising states and the next-stage PEVI value function $\widehat V_{h+1}$, there exists a filtration
$\{\mathcal F_{\tau,h}\}_{\tau\in\mathcal I_h}$ such that
$x_{\tau,h}:=\varphi(\widehat s_{\tau,h},b_{\tau,h})$ is predictable and, for every value function $V$ appearing in
the PEVI recursion,
$$\mathbb E[\xi_{\tau,h}(V) \mid \mathcal F_{\tau-1,h} ]=0,$$
with $\xi_{\tau,h}(V)$ conditionally $H$-sub-Gaussian.

\end{assumption}

\begin{assumption}[Abstraction and binning]
\label{ass:qising_abstraction_compact}
Let $\|\cdot\|_{\mathrm{TV}}$ denote total variation distance. Let
$P_h^{Z,S}(\cdot\mid z,b)$ and $r_h^{Z,S}(z,b)$ be the full-network
transition law and reward under the target rule. Let
$\phi_{\theta_0\#}P_h^{Z,S}(\cdot\mid z,b)$ be the pushforward law of
$\phi_{\theta_0}(Z_{h+1})$. For all relevant $h,z,b$,
$$ \left| r_h^{Z,S}(z,b) -  r_h^S(\phi_{\theta_0}(z),b)
  \right| \le \varepsilon_r^{\mathrm{agg}},
  \qquad
  \left\| \phi_{\theta_0\#}P_h^{Z,S}(\cdot\mid z,b) -
    P_h^S(\cdot\mid \phi_{\theta_0}(z),b) \right\|_{\mathrm{TV}}
  \le \varepsilon_P^{\mathrm{agg}}. $$
Moreover, the target and logged aggregate MDPs are close:
$$ \sup_{h,s,b} |r_h^S(s,b)-r_h^\mu(s,b)| \le
  \varepsilon_r^{\mathrm{bin}}, \qquad
  \sup_{h,s,b} \|P_h^S(\cdot\mid s,b)-P_h^\mu(\cdot\mid s,b)\|_{\mathrm{TV}}
  \le \varepsilon_P^{\mathrm{bin}}.$$
If the historical policy also randomizes uniformly within the selected bin,
then $\varepsilon_r^{\mathrm{bin}}=\varepsilon_P^{\mathrm{bin}}=0$.
\end{assumption}

\begin{assumption}[First-stage Q-Ising stability]
\label{ass:first_stage_compact}
There exists an event $\mathcal E_\theta$, with $\mathbb P(\mathcal E_\theta)\ge1-\delta_\theta$, such that on $\mathcal E_\theta$, simultaneously over all stages $h$, all value functions $V$ appearing in the PEVI recursion, and all $(s,b)\in\mathcal S\times[K]$,
\begin{equation}
\left| \varphi(s,b)^\top \Lambda_h^{-1} \sum_{\tau\in\mathcal I_h}
x_{\tau,h}e_{\tau,h}^{\theta}(V) \right| \le
c_\theta L_{\mathrm{st}}\varepsilon_\theta .
\label{eq:projected_theta_residual}
\end{equation}
Here $L_{\mathrm{st}}$ is a stability constant controlling the sensitivity of
features, PEVI value functions, and Bellman targets to perturbations of the
Q-Ising state.
\end{assumption}

\begin{assumption}[Approximate linear Bellman regression] \label{ass:linear_bellman_compact}
For every stage $h$ and every value function $V$ appearing in the PEVI recursion, there exists $w_h(V)\in\mathbb R^{d_{\phi}}$, $\|w_h(V)\|_2\le W$, such that
$$\sup_{s,b} \left| (\mathcal T_h^\mu V)(s,b)-\varphi(s,b)^\top w_h(V) \right| \le \varepsilon_{\mathrm{lin}}.$$
In addition, the Bellman approximation residual is small after projection onto
the empirical ridge geometry:
\begin{equation}
\sup_{s,b} \left| \varphi(s,b)^\top \Lambda_h^{-1} \sum_{\tau\in\mathcal I_h}
x_{\tau,h}e_{\tau,h}^{\mathrm{lin}}(V) \right| \le
c_{\mathrm{lin}}\varepsilon_{\mathrm{lin}} .
\label{eq:projected_lin_residual}
\end{equation}
Exact linear MDPs satisfy this condition with
\(\varepsilon_{\mathrm{lin}}=0\).
\end{assumption}

\paragraph{Combining Assumption \ref{ass:first_stage_compact}-\ref{ass:linear_bellman_compact}} On the event \(\mathcal E_\theta\), define
$$e_{\tau,h}(V) := e_{\tau,h}^{\mathrm{lin}}(V)
+ e_{\tau,h}^{\theta}(V).$$
Combining Assumptions~\ref{ass:first_stage_compact} and
\ref{ass:linear_bellman_compact} gives, simultaneously over all stages $h$,
all PEVI value functions $V$, and all $(s,b)$,
\begin{equation}
\left| \varphi(s,b)^\top \Lambda_h^{-1}
\sum_{\tau\in\mathcal I_h} x_{\tau,h}e_{\tau,h}(V) \right|
\le c_{\mathrm{app}} \left( \varepsilon_{\mathrm{lin}}
+L_{\mathrm{st}}\varepsilon_\theta \right),
\label{eq:appendix_projected_residual}
\end{equation}
where $c_{\mathrm{app}}=c_{\mathrm{lin}}+c_\theta$.

In summary, Assumption~\ref{ass:valid_offline_compact} is the causal and statistical validity condition needed to treat the logged blocks as usable PEVI regression samples. The cross-fitting convention makes the estimated states $\widehat s_{\tau,h}$ predictable with respect to the Bellman noise, and the stage-wise conditioning handles the randomness of $\widehat V_{h+1}$. Assumption~\ref{ass:qising_abstraction_compact} elaborates the Q-Ising network-condition that explains the cost of abstractions and binning. Assumption~\ref{ass:first_stage_compact} gives the first-stage stability condition and defines such event. Finally, Assumption~\ref{ass:linear_bellman_compact} is a projected
misspecification condition, which requires the Bellman residual to be small in the PEVI ridge regression.

\subsection{Auxiliary Lemmas}

\begin{lemma}[Finite-horizon simulation lemma]
\label{lem:simulation_final}
Consider two finite-horizon MDPs  with the same state space, action space,
horizon, and initial state $x_1$:
$$\mathcal M^m = (\mathcal X,\mathcal A,\{P_h^m\}_{h=1}^H,\{r_h^m\}_{h=1}^H,H), \qquad m\in\{1,2\},$$
Further suppose that the rewards are bounded in $(0,1)$ and for all $h,x,a$,
\begin{align*}
     &|r_h^1(x,a)-r_h^2(x,a)| \le \varepsilon_r, \quad  \|P_h^1(\cdot\mid x,a)-P_h^2(\cdot\mid x,a)\|_{\mathrm{TV}} \le \varepsilon_P
\end{align*}
Then, for any Markov policy $\pi=(\pi_1,\ldots,\pi_H)$,
$$ |V_1^{1,\pi}(x_1)-V_1^{2,\pi}(x_1)| \le H\varepsilon_r+H^2\varepsilon_P.$$
\end{lemma}

\begin{proof}
This is a standard result in offline reinforcement learning (see e.g., \citet{10.1145/1102351.1102352}). For completeness, we provide a proof tailored to our MDP structures. For $m\in\{1,2\}$, let $V_h^{m,\pi}$ be the value-to-go of policy $\pi$ in  $\mathcal M^m$ with $V_{H+1}^{m,\pi}\equiv0$. Define
$$\Delta_h := \sup_x |V_h^{1,\pi}(x)-V_h^{2,\pi}(x)|.$$

Fix $h$ and $x$ and define $a=\pi_h(x)$. By the Bellman equations, we have 
$$|V_h^{1,\pi}(x)-V_h^{2,\pi}(x)| \le |r_h^1(x,a)-r_h^2(x,a)|+ \left|\mathbb E_{P_h^1(\cdot\mid x,a)}
[V_{h+1}^{1,\pi}(X')] - \mathbb E_{P_h^2(\cdot\mid x,a)}[V_{h+1}^{2,\pi}(X')]\right|.$$

Adding and subtracting $ \mathbb E_{P_h^1(\cdot\mid x,a)}[V_{h+1}^{2,\pi}(X')]$ on the right hand side of the equation gives
\begin{align*}
|V_h^{1,\pi}(x)-V_h^{2,\pi}(x)| &\le \varepsilon_r +
\mathbb E_{P_h^1(\cdot\mid x,a)} \left[ |V_{h+1}^{1,\pi}(X')-V_{h+1}^{2,\pi}(X')| \right] \\
&\quad+ \left| \mathbb E_{P_h^1(\cdot\mid x,a)} [V_{h+1}^{2,\pi}(X')] - \mathbb E_{P_h^2(\cdot\mid x,a)} [V_{h+1}^{2,\pi}(X')] \right|.
\end{align*}
Recall that for the total variation distance and for any measurable function $f$ satisfying $0\le f\le B$, we have $ \left| \mathbb E_P[f]-\mathbb E_Q[f] \right| \le B\|P-Q\|_{\mathrm{TV}}.$ Since rewards are in $[0,1]$, $ 0\le V_{h+1}^{2,\pi}(x')\le H-h$. Hence we have
$$\Delta_h \le \varepsilon_r+\Delta_{h+1}+(H-h)\varepsilon_P.$$
Since $\Delta_{H+1}=0$, unrolling the expression above from $H$ to $1$ gives
$$\Delta_1 \le H\varepsilon_r + \sum_{h=1}^H (H-h)\varepsilon_P = H\varepsilon_r + \frac{H(H-1)} {2}\varepsilon_P \le H\varepsilon_r+H^2\varepsilon_P.$$
\end{proof}

\begin{lemma}[Q-Ising abstraction transfer] \label{lem:abstraction_transfer_final}
Under Assumption \ref{ass:qising_abstraction_compact}, for every policy $\pi\in\Pi_S$ whose induced trajectory is covered by the local bounds in Assumption \ref{ass:qising_abstraction_compact},
$$\left| W_H^Z(\pi^\circ;Z_1) - V_1^{S,\pi}(s_1^\star) \right| \le \Delta_{\mathrm{agg}},$$
$$\left|V_1^{S,\pi}(s_1^\star) - V_1^{\mu,\pi}(s_1^\star) \right| \le \Delta_{\mathrm{bin}},$$
where $\Delta_{\mathrm{agg}} := H\varepsilon_r^{\mathrm{agg}} + H^2\varepsilon_P^{\mathrm{agg}}$ and $\Delta_{\mathrm{bin}} := H\varepsilon_r^{\mathrm{bin}}+ H^2\varepsilon_P^{\mathrm{bin}}$. Consequently, we have
$$\left| W_H^Z(\pi^\circ;Z_1) - V_1^{\mu,\pi}(s_1^\star) \right| \le \Delta_{\mathrm{abs}},$$
where $ \Delta_{\mathrm{abs}} := \Delta_{\mathrm{agg}}+\Delta_{\mathrm{bin}}.$
\end{lemma}

\begin{proof}
To see the first statement, recall that for $\pi\in\Pi_S$, its oracle full-history lift is $\pi_h^\circ(z)= \pi_h(\phi_{\theta_0}(z)).$ That is, the full-network policy and the aggregate policy make the same bin decision whenever the aggregate state is $s=\phi_{\theta_0}(z)$. Under the target rule, choosing bin $b$ in the full network means drawing $a_h\sim\mathrm{Unif}(\mathcal B_b)$. On the other hand, the resulting full-network transition is $P_h^{Z,S}(\cdot\mid z,b)$, with the induced next Q-Ising state has law as $\phi_{\theta_0\#}P_h^{Z,S}(\cdot\mid z,b)$. By Assumption~\ref{ass:qising_abstraction_compact}, for all relevant $h,z,b$, we have $ \left| r_h^{Z,S}(z,b)-r_h^S(\phi_{\theta_0}(z),b) \right| \le \varepsilon_r^{\mathrm{agg}}$ and$ \left\| \phi_{\theta_0\#}P_h^{Z,S}(\cdot\mid z,b)
  - P_h^S(\cdot\mid \phi_{\theta_0}(z),b) \right\|_{\mathrm{TV}} \le \varepsilon_P^{\mathrm{agg}}.$
Applying the finite-horizon simulation lemma \ref{lem:simulation_final} to the pushed-forward full-network process and the target aggregate MDP yields the first claim.

The second statement is more straightforward since both $\mathcal M^S$ and the logged aggregate MDP $\mathcal M^\mu$ share the same state space $\mathcal S$, action space $[K]$, and horizon $H$. The statement follows by the second statement in Assumption~\ref{ass:qising_abstraction_compact} and by applying the simulation Lemama \ref{lem:simulation_final} again.
\end{proof}

\begin{lemma}[Self-normalized PEVI Noise]
\label{lem:self_normalized_pevi}
For each \(h\in[H]\), define $$S_h := \sum_{\tau\in\mathcal I_h}x_{\tau,h}\xi_{\tau,h},
  \qquad
  \Lambda_h = \lambda I_{d_\phi} +
  \sum_{\tau\in\mathcal I_h}x_{\tau,h}x_{\tau,h}^\top .$$
Suppose $\|x_{\tau,h}\|_2\le1$, $x_{\tau,h}$ is predictable with respect to the stage-$h$ sample filtration, and $\xi_{\tau,h}$ is conditionally mean-zero and $H$-sub-Gaussian. Then, with probability at least $1-\delta$,
simultaneously for all \(h\in[H]\),
\begin{equation}
  \|S_h\|_{\Lambda_h^{-1}} \le C H \sqrt{ d_\phi \log\!\left( \frac{H(1+n_{\log}/\lambda)}{\delta} \right)
  } ,
  \label{eq:self_normalized_bound}
\end{equation}
where \(C>0\) is a universal constant. Consequently, for every $x=\varphi(s,b)$,
\begin{equation}
  \left| x^\top\Lambda_h^{-1}S_h \right| \le
  C H \sqrt{ d_\phi \log\!\left( \frac{H(1+n_{\log}/\lambda)}{\delta} \right)} \sqrt{x^\top\Lambda_h^{-1}x}.
  \label{eq:self_normalized_prediction_bound}
\end{equation}
\end{lemma}

\begin{proof}
Fix a stage $h$. By the standard self-normalized martingale inequality (see Theorem 1 of \cite{abbasi2011improved}), we have for any \(\delta_h\in(0,1)\),
$$\|S_h\|_{\Lambda_h^{-1}} \le H \sqrt{ 2\log\left( \frac{ \det(\Lambda_h)^{1/2}}{\det(\lambda I_{d_\phi})^{1/2}\delta_h}\right)}$$
with probability at least $1-\delta_h$. Taking $\delta_h=\delta/H$ and applying the union bound gives the same display simultaneously over all stages with probability at least $1-\delta$.

It remains to simplify the determinant term. Let $A_h:=\sum_{\tau\in\mathcal I_h}x_{\tau,h}x_{\tau,h}^\top$ so that $\Lambda_h=\lambda I_{d_\phi}+A_h.$ Denote $\nu_1,\ldots,\nu_{d_\phi}$ as the eigenvalues of $A_h$. It then follows that
$$ \frac{\det(\Lambda_h)}{\det(\lambda I_{d_\phi})} = \det\left(I_{d_\phi}+\frac{A_h}{\lambda}\right)
  =\prod_{j=1}^{d_\phi} \left(1+\frac{\nu_j}{\lambda}\right).$$
Since $\|x_{\tau,h}\|_2\le1$, we have
$$\sum_{j=1}^{d_\phi}\nu_j = \operatorname{tr}(A_h) = \sum_{\tau\in\mathcal I_h}\|x_{\tau,h}\|_2^2
  \le n_h \le n_{\log}.$$
Recall that the AM-GM inequality gives $\frac{\sum_{j=1}^n a_j}{n} \geq (\prod_{j=1}^n a_j)^{1/n}$ for nonnegative $a_j$. An application of AM-GM inequality with $a_j:=\frac{\nu_j}{\lambda}$ yields
$$\prod_{j=1}^{d_\phi} \left(1+\frac{\nu_j}{\lambda}\right) \le \left( 1+\frac{n_h}{\lambda d_\phi} \right)^{d_\phi} \le
  \left( 1+\frac{n_{\log}}{\lambda d_\phi} \right)^{d_\phi}.$$
It follows that 
$$\|S_h\|_{\Lambda_h^{-1}} \le H \sqrt{ d_\phi \log\left( 1+\frac{n_{\log}}{\lambda d_\phi} \right)
    + 2\log(H/\delta)}.$$
Absorbing constants in the log gives
$$\|S_h\|_{\Lambda_h^{-1}} \le C H \sqrt{ d_\phi \log\!\left( \frac{H(1+n_{\log}/\lambda)}{\delta}\right)}.$$
This proves the first statement.

The second statement follows by Cauchy-Schwarz:
$$ |x^\top\Lambda_h^{-1}S_h| = \left| x^\top\Lambda_h^{-1/2}\Lambda_h^{-1/2}S_h
  \right| \le \sqrt{x^\top\Lambda_h^{-1}x}\, \|S_h\|_{\Lambda_h^{-1}}.$$
Combining the inequality above with the first statement proves the result.
\end{proof}

\begin{lemma}[PEVI Bellman confidence]
\label{lem:pevi_confidence_final}
Suppose Assumptions~\ref{ass:valid_offline_compact}, \ref{ass:first_stage_compact}, and \ref{ass:linear_bellman_compact} hold with the residual conditions described in
\eqref{eq:regression_residual_decomposition}-\eqref{eq:appendix_projected_residual}. Define
$$ \Gamma_h(s,b) := \beta \sqrt{ \varphi(s,b)^\top \Lambda_h^{-1} \varphi(s,b)}.$$
Choose
\begin{equation}
  \beta = C_\beta \left[ H \sqrt{ d_\phi \log\!\left( \frac{H(1+n_{\log}/\lambda)}{\delta}
    \right)}
  + \sqrt{\lambda}\,W \right]
  \label{eq:appendix_beta_choice}
\end{equation}
for a sufficiently large universal constant $C_\beta$. Then, with
probability at least $1-\delta-\delta_\theta$, for every $h\in[H]$ and
every $(s,b)\in\mathcal S\times[K]$,
\begin{equation}
  \left| \varphi(s,b)^\top\widehat w_h - (\mathcal T_h^\mu\widehat V_{h+1})(s,b)
  \right| \le \Gamma_h(s,b) +
  \zeta,
  \label{eq:pevi_confidence_final}
\end{equation}
where 
\begin{equation}
  \zeta = C_{\mathrm{app}} \left( \varepsilon_{\mathrm{lin}} +
    L_{\mathrm{st}}\varepsilon_\theta \right)
  \label{eq:zeta_final}
\end{equation}
for a universal constant $C_{\mathrm{app}}>0$.
\end{lemma}

\begin{proof}
Let $\mathcal E_{\mathrm{SN}}$ be the self-normalized concentration event
from Lemma~\ref{lem:self_normalized_pevi}, and $\mathcal E_\theta$ be the
event on which the projected first-stage residual bound \eqref{eq:projected_theta_residual} holds. We will work on $ \mathcal E_{\mathrm{SN}}\cap\mathcal E_\theta.$

Fix a stage \(h\). For notational simplicity, write $x_\tau := x_{\tau,h} =\varphi(\widehat s_{\tau,h},b_{\tau,h})$, $x:=\varphi(s,b)$, and $ \Lambda_h = \lambda I_{d_\phi} +
  \sum_{\tau\in\mathcal I_h}x_\tau x_\tau^\top.$ Recall from the regression decomposition in  \eqref{eq:regression_residual_decomposition} that we have
$$Y_{\tau,h}= r_{\tau,h} + \widehat V_{h+1}(\widehat s_{\tau,h+1}) =
  x_\tau^\top w_h + \xi_{\tau,h} +e_{\tau,h}^{\mathrm{lin}} + e_{\tau,h}^{\theta},$$
  where we write $V=\widehat V_{h+1}$ and $ w_h:= w_h(\widehat V_{h+1})$.
  Let $ e_{\tau,h} := e_{\tau,h}^{\mathrm{lin}} + e_{\tau,h}^{\theta}$, then by the ridge regression formula, we have $\widehat w_h = \Lambda_h^{-1} \sum_{\tau\in\mathcal I_h} x_\tau Y_{\tau,h}$. Substituting the target decomposition yields
  $$ \widehat w_h-w_h =  -\lambda\Lambda_h^{-1}w_h + \Lambda_h^{-1}\sum_{\tau\in\mathcal I_h}x_\tau\xi_{\tau,h} + \Lambda_h^{-1}\sum_{\tau\in\mathcal I_h}x_\tau e_{\tau,h}.$$
  It follows that
\begin{align*}
  &|x^\top(\widehat w_h-w_h)| \le \lambda|x^\top\Lambda_h^{-1}w_h| + \left| x^\top\Lambda_h^{-1}
  \sum_{\tau\in\mathcal I_h}x_\tau\xi_{\tau,h} \right| +
  \left| x^\top\Lambda_h^{-1} \sum_{\tau\in\mathcal I_h}x_\tau e_{\tau,h} \right|.
\end{align*}

We will bound these three terms separately. Specifically, the first term can by bound by Cauchy-Schwarz and the fact that $\Lambda_h\succeq \lambda I_{d_\phi}$:
\begin{align}
    & \lambda|x^\top\Lambda_h^{-1}w_h| = \lambda \left|  x^\top\Lambda_h^{-1/2}\Lambda_h^{-1/2}w_h \right| \le \lambda \sqrt{x^\top\Lambda_h^{-1}x} \sqrt{w_h^\top\Lambda_h^{-1}w_h}, \\
    &   w_h^\top\Lambda_h^{-1}w_h \le \lambda^{-1}\|w_h\|_2^2,
\end{align}
which implies
$$\lambda|x^\top\Lambda_h^{-1}w_h| \le \sqrt{\lambda}W \sqrt{x^\top\Lambda_h^{-1}x}.$$

For the second term, we apply Lemma~\ref{lem:self_normalized_pevi} to obtain
$$ \left|  x^\top\Lambda_h^{-1} \sum_{\tau\in\mathcal I_h}x_\tau\xi_{\tau,h}
  \right| \le C H \sqrt{ d_\phi \log\!\left( \frac{H(1+n_{\log}/\lambda)}{\delta}
    \right)} \sqrt{x^\top\Lambda_h^{-1}x}.$$

We bound the last term by \eqref{eq:projected_lin_residual} and \eqref{eq:projected_theta_residual}, which gives
\begin{align*}
\left| x^\top\Lambda_h^{-1} \sum_{\tau\in\mathcal I_h}x_\tau e_{\tau,h}  \right| 
& \le \left| x^\top\Lambda_h^{-1} \sum_{\tau\in\mathcal I_h} x_\tau e_{\tau,h}^{\mathrm{lin}}  \right| +  \left| x^\top\Lambda_h^{-1} \sum_{\tau\in\mathcal I_h} x_\tau e_{\tau,h}^{\theta} \right| \\
& \le c_{\mathrm{lin}}\varepsilon_{\mathrm{lin}} + c_\theta L_{\mathrm{st}}\varepsilon_\theta
  \le c_{\mathrm{app}} \left( \varepsilon_{\mathrm{lin}} + L_{\mathrm{st}}\varepsilon_\theta
  \right).
\end{align*}
Combining the last three displays and choosing $C_\beta$ sufficiently large in \eqref{eq:appendix_beta_choice} yields
\begin{equation*}
  |x^\top(\widehat w_h-w_h)| \le \Gamma_h(s,b) + c_{\mathrm{app}}
  \left( \varepsilon_{\mathrm{lin}} + L_{\mathrm{st}}\varepsilon_\theta
  \right).
\end{equation*}
By assumption \ref{ass:linear_bellman_compact}, we have $\left|
  x^\top w_h - (\mathcal T_h^\mu\widehat V_{h+1})(s,b) \right| \le \varepsilon_{\mathrm{lin}}$. Combining this with the equation above, and absorbing constants into $C_{\mathrm{app}}$ gives 
$$\left| \varphi(s,b)^\top\widehat w_h  - (\mathcal T_h^\mu\widehat V_{h+1})(s,b) \right|
  \le \Gamma_h(s,b) +  C_{\mathrm{app}} \left( \varepsilon_{\mathrm{lin}} + L_{\mathrm{st}}\varepsilon_\theta \right).$$
Finally, the probability $ 1-\delta-\delta_\theta$ statement follows by union bound.
\end{proof}

\paragraph{Remark for Lemma \ref{lem:pevi_suboptimality_mu_final}}: The following lemma is proved entirely inside the oracle aggregate MDP $\mathcal M^\mu
  =(\mathcal S,[K],\{P_h^\mu\}_{h=1}^H,\{r_h^\mu\}_{h=1}^H,H)$, in which the nodes does not necessarily have to be selected uniformly from the bin. For a bounded function $V:\mathcal S\to[0,H]$, recall that the Bellman equation gives $(\mathcal T_h^\mu V)(s,b)  = r_h^\mu(s,b) + \mathbb E_{s'\sim P_h^\mu(\cdot\mid s,b)} [V(s')]$. Additionally, the value-to-go in $\mathcal M^\mu$ is given by 
$$ V_h^{\mu,\pi}(s)  = \mathbb E_{\mathcal M^\mu}^{\pi} \left[
    \sum_{u=h}^H r_u^\mu(s_u,b_u) \,\middle|\, s_h=s \right],
  \qquad V_{H+1}^{\mu,\pi} := 0. $$
Thus, if $b=\pi_h(s)$, the policy Bellman equation is $ V_h^{\mu,\pi}(s)
  = (\mathcal T_h^\mu V_{h+1}^{\mu,\pi})(s,b).$

\begin{lemma}[PEVI Optimality in Aggregate MDP] \label{lem:pevi_suboptimality_mu_final}
Assume the PEVI Bellman confidence event \eqref{eq:pevi_confidence_final} from Lemma~\ref{lem:pevi_confidence_final} holds. Then for every comparator policy $\pi\in\Pi_S$,  we have
\begin{equation}
  V_1^{\mu,\pi}(s_1^\star) - V_1^{\mu,\widehat\pi}(s_1^\star)
  \le 2\beta\, \mathfrak U_{\mathcal D}^{\mu}(\pi) + 2H\zeta ,
  \label{eq:pevi_mu_suboptimality_detailed}
\end{equation}
where $ \widehat\pi_h(s) \in \arg\max_{b\in[K]}\widehat Q_h(s,b)$ and $\mathfrak U_{\mathcal D}^{\mu}(\pi) = \sum_{h=1}^{H} \mathbb E_{\mathcal M^\mu}^{\pi}
  \left[ \sqrt{  \varphi(s_h,b_h)^\top \Lambda_h^{-1} \varphi(s_h,b_h)
  }\right]$. Consequently, with $\zeta = C_{\mathrm{app}} \left( \varepsilon_{\mathrm{lin}} + L_{\mathrm{st}}\varepsilon_\theta \right)$, we have
$$ V_1^{\mu,\pi}(s_1^\star) - V_1^{\mu,\widehat\pi}(s_1^\star)
  \le 2\beta\, \mathfrak U_{\mathcal D}^{\mu}(\pi) +
  C H \left( \varepsilon_{\mathrm{lin}} + L_{\mathrm{st}}\varepsilon_\theta
  \right).$$
\end{lemma}

\begin{proof}
First observe that to show \eqref{eq:pevi_mu_suboptimality_detailed}, it is sufficient to prove the following two inequalities:
\begin{align}
    &  V_1^{\mu,\pi}(s_1^\star)-\widehat V_1(s_1^\star) \le 2\beta\,\mathfrak U_{\mathcal D}^{\mu}(\pi) + H\zeta. \label{eq:first-ineq} \\
    &  \widehat V_1(s_1^\star)  - V_1^{\mu,\widehat\pi}(s_1^\star) \le H\zeta.
\end{align}

\paragraph{First Inequality} Fix a comparator policy $\pi\in\Pi_S$ and define
$$ D_h^\pi(s) := V_h^{\mu,\pi}(s)-\widehat V_h(s).$$
Since $ \widehat V_h(s) = \max_{a\in[K]}\widehat Q_h(s,a)$ and $ V_h^{\mu,\pi}(s) = (\mathcal T_h^\mu V_{h+1}^{\mu,\pi})(s,b)$, we have $  D_h^\pi(s)
  \le (\mathcal T_h^\mu V_{h+1}^{\mu,\pi})(s,b) - \widehat Q_h(s,b).$ Adding and substracting $\mathcal T_h^\mu \widehat V_{h+1})(s,b)$ on the right hand side gives
 $$ D_h^\pi(s) \le \left[ (\mathcal T_h^\mu\widehat V_{h+1})(s,b)
  - \widehat Q_h(s,b) \right] + \mathbb E_{s'\sim P_h^\mu(\cdot\mid s,b)} [
    D_{h+1}^\pi(s')].$$ 
To upper bound $\left[ (\mathcal T_h^\mu\widehat V_{h+1})(s,b)
  - \widehat Q_h(s,b) \right]$, recall that (1) the Bellman confidence event from Lemma \ref{lem:pevi_confidence_final} gives $\left| \varphi(s,b)^\top\widehat w_h - (\mathcal T_h^\mu\widehat V_{h+1})(s,b) \right| \le \Gamma_h(s,b) +
  \zeta$, and (2) that $\widehat Q_h(s,b) = \left[ \varphi(s,b)^\top\widehat w_h
    - \Gamma_h(s,b) \right]_{[0,H-h+1]}$. If $\widehat  Q_h(s,b)$ is not clipped, it is straightforward to see combining these two facts yields
    $$(\mathcal T_h^\mu\widehat V_{h+1})(s,b) - \widehat Q_h(s,b) \le 2\Gamma_h(s,b)+\zeta.$$
    It is also easy to verify the inequality above still holds even if $ \widehat Q_h(s,b)$ is clipped at $0$ or at $(H-h+1)$. It follows that we have
    \begin{equation} \label{eq:unroll-eq}
         D_h^\pi(s) \le 2\Gamma_h(s,\pi_h(s)) + \mathbb E_{s'\sim P_h^\mu(\cdot\mid s,\pi_h(s))} [ D_{h+1}^\pi(s') ] + \zeta.
    \end{equation}

    We then unroll \eqref{eq:unroll-eq} along the trajectory generated by  policy $\pi$ in $\mathcal M^\mu$. Starting at $s_1^\star$, let $b_h=\pi_h(s_h)$, $s_{h+1}\sim P_h^\mu(\cdot\mid s_h,b_h)$. Since both $ V_{H+1}^{\mu,\pi}:=0$ and $\widehat V_{H+1}:=0$, we have $D_{H+1}^\pi=0$. Repeatably applying \eqref{eq:unroll-eq} from $h=1$ to $H$ gives
    $$V_1^{\mu,\pi}(s_1^\star)-\widehat V_1(s_1^\star)  = D_1^\pi(s_1^\star)
    \le 2 \sum_{h=1}^H \mathbb E_{\mathcal M^\mu}^{\pi}[ \Gamma_h(s_h,b_h)] + H\zeta.$$
    The first inequality therefore follows by recalling the expression of $\mathfrak U_{\mathcal D}^{\mu}$.
    
\paragraph{Second Inequality} To prove the second inequality, it is sufficient to prove that for every $h\in[H+1]$ and every $s\in\mathcal S$, we have
$$\widehat V_h(s)-V_h^{\mu,\widehat\pi}(s) \le (H-h+1)\zeta.$$
We will prove this via backward induction. It is clear to see the base case $H+1$ holds because $\widehat V_{H+1} = V_{H+1}^{\mu,\widehat\pi} =0$. Suppose this pattern holds at stage $h+1$. Fix $s\in\mathcal S$, and let $ \widehat b=\widehat\pi_h(s)$. Note that $\widehat V_h(s)=\widehat Q_h(s,\widehat b)$ by exact greediness. Again recall the Bellman confidence event from Lemma \ref{lem:pevi_confidence_final}: $\left| \varphi(s,b)^\top\widehat w_h - (\mathcal T_h^\mu\widehat V_{h+1})(s,b) \right| \le \Gamma_h(s,b) +
  \zeta$ and the fact that $\widehat Q_h(s,b) = \left[ \varphi(s,b)^\top\widehat w_h - \Gamma_h(s,b) \right]_{[0,H-h+1]}$. Combining these two facts gives
$$\widehat Q_h(s,\widehat b)  \le (\mathcal T_h^\mu\widehat V_{h+1})(s,\widehat b)+\zeta.$$
Given that $\widehat V_h(s)=\widehat Q_h(s,\widehat b)$, expanding the right side with the Bellman equation gives
\begin{align*}
    \widehat V_h(s) &\le r_h^\mu(s,\widehat b) + \mathbb E_{s'\sim P_h^\mu(\cdot\mid s,\widehat b)} [\widehat V_{h+1}(s')] + \zeta.\\
    & \le  r_h^\mu(s,\widehat b) + \mathbb E_{s'\sim P_h^\mu(\cdot\mid s,\widehat b)} [ V_{h+1}^{\mu,\widehat\pi}(s') ] + (H-h)\zeta+\zeta, \\
    & =  V_h^{\mu,\widehat\pi}(s) + (H-h+1)\zeta.
\end{align*}
where the second inequality is due to induction hypothesis. Evaluating at $h=1$ finishes the proof for the second inequality.
\end{proof}

\subsection{Proof of Theorem \ref{thm:pevi_qising_compact}}
\begin{proof}
 Let $\beta = C_\beta \left[ H \sqrt{ d_{\phi}\log\!\left( \frac{H(1+n_{\log}/\lambda)}{\delta} \right)} +\sqrt{\lambda} W \right]$ so that we can work on the PEVI Bellman confidence event $\mathcal E_{\mathrm{conf}}$ in Lemma \ref{lem:pevi_confidence_final}, where $\mathbb P(\mathcal E_{\mathrm{conf}}) \ge 1-\delta-\delta_\theta .$ We will work on this event throughout the proof.
 
 We first apply the first part of Lemma \ref{lem:abstraction_transfer_final} to obtain:
 \begin{align*}
      W_H^Z((\pi_S^\star)^\circ;Z_1) &\le V_1^{S,\pi_S^\star}(s_1^\star) + \Delta_{\mathrm{agg}}, \\
      W_H^Z(\widehat\pi^\circ;Z_1) &\ge  V_1^{S,\widehat\pi}(s_1^\star) - \Delta_{\mathrm{agg}},
 \end{align*}
 which implies that
$$W_H^Z((\pi_S^\star)^\circ;Z_1) - W_H^Z(\widehat\pi^\circ;Z_1)
  \le  V_1^{S,\pi_S^\star}(s_1^\star) - V_1^{S,\widehat\pi}(s_1^\star)
  + 2\Delta_{\mathrm{agg}}.$$
 Furthermore, an application of the second part of Lemma \ref{lem:abstraction_transfer_final} yields:
 \begin{align*}
      V_1^{S,\pi_S^\star}(s_1^\star) &\le V_1^{\mu,\pi_S^\star}(s_1^\star) +
  \Delta_{\mathrm{bin}}, \\
     V_1^{S,\widehat\pi}(s_1^\star) &\ge V_1^{\mu,\widehat\pi}(s_1^\star)  -
  \Delta_{\mathrm{bin}}.
 \end{align*}
 It follows that
 $$W_H^Z((\pi_S^\star)^\circ;Z_1) - W_H^Z(\widehat\pi^\circ;Z_1)
  \le V_1^{\mu,\pi_S^\star}(s_1^\star) - V_1^{\mu,\widehat\pi}(s_1^\star)
  + 2\Delta_{\mathrm{agg}} + 2\Delta_{\mathrm{bin}}.$$

  Since we work on the event $\mathcal E_{\mathrm{conf}}$, we can apply Lemma \ref{lem:pevi_suboptimality_mu_final} to obtain 
  $$V_1^{\mu,\pi_S^\star}(s_1^\star) - V_1^{\mu,\widehat\pi}(s_1^\star)
  \le 2\beta\, \mathfrak U_{\mathcal D}^{\mu}(\pi_S^\star)
  + C H \left( \varepsilon_{\mathrm{lin}} + L_{\mathrm{st}}\varepsilon_\theta \right).$$
  Recall that $ \Delta_{\mathrm{abs}} = \Delta_{\mathrm{agg}} + \Delta_{\mathrm{bin}}$. Combining the last two inequalities with sufficient large constant $C$ yields
  $$W_H^Z((\pi_S^\star)^\circ;Z_1) - W_H^Z(\widehat\pi^\circ;Z_1)
  \le 2\beta\, \mathfrak U_{\mathcal D}^{\mu}(\pi_S^\star)  +
  C \left[ \Delta_{\mathrm{abs}} +  H\varepsilon_{\mathrm{lin}}
    + HL_{\mathrm{st}}\varepsilon_\theta
  \right].$$
  This finishes the proof.
\end{proof}

\section{Experimental Set-up Details}
\label{app:experiment}
\subsection{SIS Dynamics}
\label{app:sis}

We model information diffusion on a fixed, undirected graph $G = (V, E)$ with $|V| = n$ nodes.
Each node $i \in V$ is at every period in one of two states: susceptible~(S) or infected/adopted~(I).
The planner selects a bin $a_t \in \{0, \ldots, K-1\}$ at each period $t$, and the environment
transitions through three ordered sub-steps: \emph{churn}, \emph{seeding}, and \emph{spreading}.

Each currently adopted node $i$ independently reverts to susceptible with a node-specific
probability $\delta_i \in (0,1)$, capturing heterogeneous loss of interest or product abandonment.

After churn, one susceptible node is drawn uniformly at random from the chosen bin $a_t$ and
forced to adopt (treatment is ``perfect,'' i.e.\ adoption occurs with probability one).

Every currently adopted node $i$ (including the newly seeded one) independently attempts to
transmit adoption to each susceptible neighbor $j \in \mathcal{N}(i)$.
Transmission from node $i$ to node $j$ succeeds with a node-specific probability $\beta_i \in (0,1)$.
Because $j$ may have multiple infected neighbors, the probability that $j$ adopts in this period is
\begin{equation}
  p_{j}^{\mathrm{adopt}} \;=\; 1 -\prod_{i \in \mathcal{N}(j),\, i \text{ adopted}} (1 - \beta_i).
\end{equation}
\subsection{CQL Design Choices}
\label{app:design}
We use the CQL algorithm implemented by \cite{seno2022d3rlpy}. The learning rate for both the encoder and Q-network is $3 \times 10^{-4}$.
Batch normalisation is applied within the encoder, together with a dropout rate of $0.3$.
Early stopping monitors the per-epoch TD loss and halts training if improvement is smaller than
$10^{-4}$ for 10 consecutive epochs (\texttt{patience} = 10, \texttt{min\_delta} = $10^{-4}$). For hyperparameters of the CQL we use $[256, 256]$ hidden layers, the batch size is 64, and the maximum allowed number of steps is 30,000 while, steps per epoch are 1000. For conservative penalty we use $\alpha = 0.1$ and for discount we use $\psi = 0.8$ for all the experiments. 
All the experiments are seeded appropriately for reproducibility.

\subsection{Microfinance Villages Experiment Further Details}
\label{app:village_exp_details}
In this experiment, the planner aims to learn a dynamic policy for selecting which communities (“bins”) within a village to target. These communities can be interpreted as clusters of households (e.g., friend groups) obtained via edge-betweennes based community detection. We use the implementation in $\mathrm{igraph}$ library in Python. To avoid very small groups, any identified community with fewer than 10 households is merged into the largest community. 42 of the 43 microfinance villages has at least 2 identifies communities and we experiment in these 42 villages. Some community examples and the distribution of households in all the villages can be seen in Figure ~\ref{fig:village_details}

\begin{figure}[H]
    \centering
    \includegraphics[width=1\linewidth]{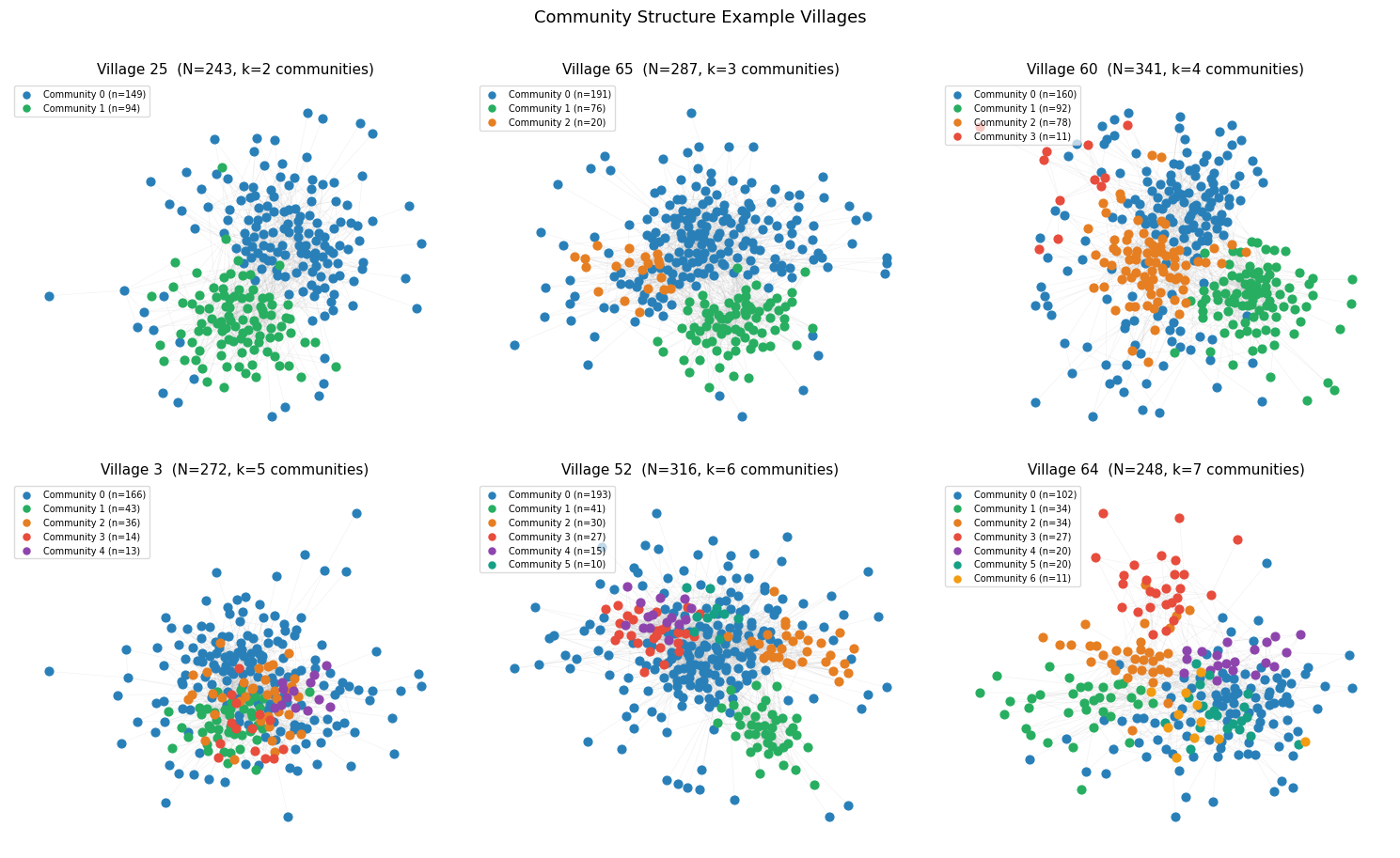}
      \includegraphics[width=1\linewidth]{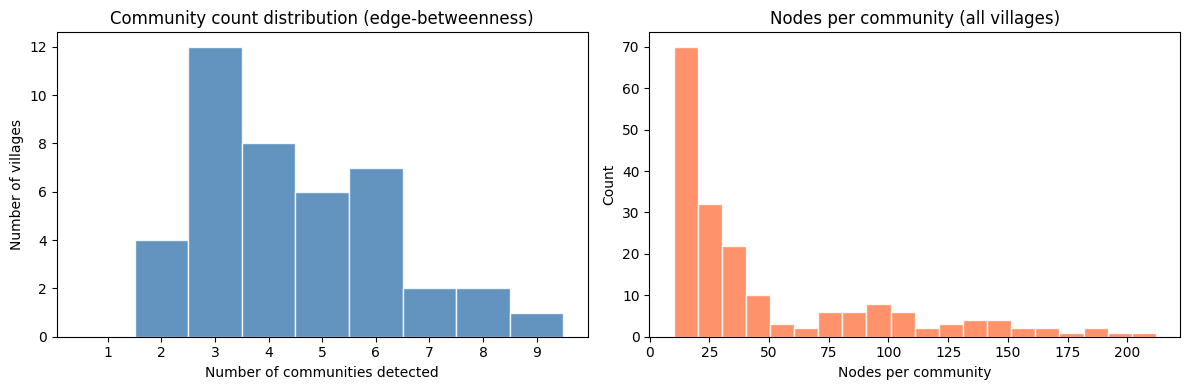}
    \caption{Top: Community detection examples from Indian Village dataset. Bottom: Distribution of identified community numbers and average degree per community }
    \label{fig:village_details}
\end{figure}

Each adopted household independently influences each of its neighbors to adopt with a community-specific probability. The spread parameters are set (in order of community size) to
$[0.01,0.5,0.05,0.07,0.06,0.02,0.01,0.4,0.1,0.3]$,
and the corresponding churn rates (probability that an adopted household stops using the product in a period) are
$[0.5,0.9,0.9,0.6,0.5,0.5,0.7,0.6,0.5,0.8]$. This environment induces substantial heterogeneity across communities, including both supercritical and subcritical diffusion regimes, as well as relatively high churn rates. As a result, learning an effective policy requires balancing expansion into new communities with maintenance of existing adoption.

For each village, the planner is provided with 500 periods of historical data generated by a random policy that selects communities uniformly at random. The evaluation (test) phase begins from a zero-adoption state. This can be interpreted as observing past adoption data from a similar product and then deploying a new, related product using a learned policy. The test horizon is 25 periods, and performance is measured by the average adoption rate per period over this horizon. Each test is independently ran 50 times to produce accurate standard deviation estimates.

The ensemble alternative for this experiment uses the same hyperparameters as defined in Section ~\ref{sec:method}. 200 MCMC draws are used for posterior parameter estimation, with 300 tune-in iterations. For the ensemble policy 20 agents are trained.
\end{document}